\documentclass{aims} 
\usepackage{amsmath}
\usepackage{paralist}
\usepackage{epsfig} 
\usepackage{epstopdf} 
\usepackage[colorlinks=true]{hyperref}
\hypersetup{urlcolor=blue, citecolor=red}
\allowdisplaybreaks

\def\currentvolume{X}
 \def\currentissue{X}
  \def\currentyear{200X}
   \def\currentmonth{XX}
    \def\ppages{X--XX}
     \def\DOI{10.3934/xx.xxxxxxx}

\newtheorem{theorem}{Theorem}[section]

\newtheorem{proposition}{Proposition}

\theoremstyle{definition}
\newtheorem{definition}[theorem]{Definition}

\usepackage{amssymb}
\usepackage{mathtools}
\usepackage{blkarray, bigstrut}
\usepackage{bm}
\usepackage{float} 

\newtheorem{assumption}{Assumption}

\newcommand{\valpha}{\mbox{\boldmath $\alpha$}}
\newcommand{\vnu}{\mbox{\boldmath $\nu$}}
\newcommand{\vtau}{\mbox{\boldmath $\tau$}}
\newcommand{\vomega}{\mbox{\boldmath $\omega$}}
\newcommand{\vGamma}{\mathbf \Gamma}
\newcommand{\vDelta}{\mathbf \Delta}
\newcommand{\va}{\mathbf a}
\newcommand{\vb}{\mathbf b}
\newcommand{\vc}{\mathbf c}
\newcommand{\vd}{\mathbf d}
\newcommand{\vf}{\mathbf f}
\newcommand{\vh}{\mathbf h}
\newcommand{\vm}{\mathbf m}
\newcommand{\vo}{\mathbf o}
\newcommand{\vq}{\mathbf q}
\newcommand{\vr}{\mathbf r}
\newcommand{\vs}{\mathbf s}

\newcommand{\vu}{\mathbf u}
\newcommand{\vv}{\mathbf v}
\newcommand{\vC}{\mathbf C}
\newcommand{\vD}{\mathbf D}
\newcommand{\vF}{\mathbf F}
\newcommand{\vG}{\mathbf G}
\newcommand{\vH}{\mathbf H}
\newcommand{\vM}{\mathbf M}
\newcommand{\vR}{\mathbf R}
\newcommand{\vS}{\mathbf S}
\newcommand{\vU}{\mathbf U}
\newcommand{\vX}{\mathbf X}

\newcommand{\cB}{\mathcal{B}}
\newcommand{\cE}{\mathcal{E}}
\newcommand{\cI}{\mathcal{I}}
\newcommand{\cJ}{\mathcal{J}}

\newcommand{\ls}{\hspace{0em}}      
\newcommand{\pc}{{\phantom A}}      
\newcommand{\bmv}{{\bm v}}          
\newcommand{\bma}{{\bm a}}          
\newcommand{\bmomega}{{\bm \omega}} 
\newcommand{\bts}{{\bar\times^*}}   
\newcommand{\bbM}{\mathbb{M}}       
\newcommand{\rmv}{{\rm v}}          
\newcommand{\R}{\mathbb{R}}     
\newcommand{\SO}{\textrm{SO}}   
\newcommand{\so}{\mathfrak{so}} 
\newcommand{\SE}{\textrm{SE}}   
\newcommand{\se}{\mathfrak{se}} 
\DeclareMathOperator{\Ad}{Ad}   
\DeclareMathOperator{\ad}{ad}   
\DeclareMathOperator{\D}{D}     

\newcommand{\tp}{\tilde{\partial}} 

\newcommand{\paragraphbold}[1]{\textbf{#1}}

\title[Sensitivity Analysis of Moving-Base Systems]
{Efficient Geometric Linearization \\ of 
       Moving-Base Rigid Robot Dynamics} 

\author[M.P. Bos, S. Traversaro, D. Pucci and A. Saccon]{}


 \email{martijn.paulus.bos@outlook.com}
 \email{silvio.traversaro@iit.it}
 \email{daniele.pucci@iit.it}
 \email{a.saccon@tue.nl}

\thanks{$^*$ Corresponding author: Alessandro Saccon}
\thanks{$^1$ This work was partly performed while the author was affiliated to the Eindhoven University of Technology}

\begin{document}
\maketitle

\centerline{\scshape Martijn Bos$^1$}
\medskip
{\footnotesize
 \centerline{Smart Robotics}
   \centerline{De Maas 8, 5684 PL Best, the Netherlands}
} 

\medskip

\centerline{\scshape Silvio Traversaro, Daniele Pucci}
\medskip
{\footnotesize
 \centerline{Artificial Mechanical Intelligence research line, Istituto Italiano di Tecnologia}
   \centerline{Via S. Quirico 19D, 16163 Genoa, Italy}
}

\medskip

\centerline{\scshape Alessandro Saccon$^*$}
\medskip
{\footnotesize
 \centerline{Department of Mechanical Engineering, Eindhoven University of Technology}
   \centerline{Groene Loper 3, PO Box 513, 5600 MB Eindhoven, the Netherlands}
}

\bigskip
\centerline{\emph{
Dedicated to Professor Tony Bloch on the occasion of his 65th birthday}
}
\bigskip

 \centerline{(Communicated by the associate editor name)}


\begin{abstract}
The linearization of the equations of motion 
of a robotics system
about a given state-input trajectory, including a controlled equilibrium state, is a valuable tool 
for model-based planning, closed-loop control, gain tuning, and state estimation. 
Contrary to the case of fixed based manipulators
with prismatic or rotary joints, the state space of moving-base robotic systems such as humanoids, quadruped robots, or aerial manipulators 
cannot be globally parametrized by a finite number of independent coordinates. 
This impossibility is a direct consequence of the fact that the state of these systems includes the system's global orientation, formally described as an element of the special orthogonal group SO(3).
As a consequence, obtaining the linearization of the equations of motion for these systems is typically resolved, from a practical perspective, by locally parameterizing the system's attitude by means of, e.g., Euler or Cardan angles.
This has the drawback, however, of introducing artificial parameterization singularities and extra derivative computations. 
In this contribution, we show that it is actually possible to define a notion of linearization that does not require the use of a local parameterization for the system's orientation, obtaining a mathematically elegant, recursive, and singularity-free linearization for moving-based robot systems.
Recursiveness, in particular, is obtained by proposing a nontrivial modification of
existing recursive algorithms
to allow for 
computations
of the geometric derivatives of the inverse dynamics and the inverse of the mass matrix of the robotic system.
The correctness of the proposed algorithm is validated by means of a numerical comparison with the result obtained via geometric finite difference.
\end{abstract}


\section{Introduction}
\label{chp:introduction}

This section provides the motivation, literature review, and contribution of this book chapter, as well as the chapter outline.

\subsection{Motivation}\label{sec:motivation}

Due to increase of computational power
in combination with advances in
computational efficiency of dynamics and optimization solvers, 
the model-based control of 
complex robot systems
such as 
humanoids and quadrupeds
is
more and more 
making use of
advanced methods such as
optimal whole-body control 
\cite{mason2014full, pucci2016automatic, marco2016automatic, Kheddar2019comanoid},
model predictive control (MPC), and 
\cite{koenemann2015whole, geoffroy2014inverse, farshidian2017real}, 
offline/online trajectory planning 
\cite{posa2014direct, tassa2012synthesis}. 

Numerical optimization methods, such as optimal planners or MPC strategies, in particular, 
typically require the computation of the sensitivity of state trajectory
with respect to variation of the input or initial conditions (see, e.g., \cite{Docquier2019multibody} and reference 
therein for optimal control of multibody systems), other than
the efficient computation of the dynamics. This 
entails the computation of
the linearization of the control vector field with respect the state and input variables, about a given trajectory.

When dealing with the dynamical system  
\begin{equation}
    \dot{x}(t) = f(x,u,t),
\end{equation}
whose state evolves on a vector space, namely, with state $x(t) \in \R^n$, input $u(t) \in \R^m$, time $t \in \R$, and 
control vector field
$f : \R^n \times \R^m \times \R \rightarrow \R^n$, that 
the sensitivity of $f$ with respect to $x$ and $u$ about the nominal trajectory $\eta(x(t),u(t))$ is straightforward to compute. Namely, it is given by
\begin{equation} \label{eq:introlinearsystem}
    \dot{z}(t) = A(\eta,t)z(t) + B(\eta,t)w(t) ,
\end{equation}
where $z(t) \in \R^n$ is the perturbation vector, $w(t) \in \R^m$ the perturbed input vector, $A(\eta,t) \in \R^{n \times n}$ the state matrix and $B(\eta,t) \in \R^{n \times m}$ the input matrix \cite[Section 3.3]{khalil2002nonlinear}.
For \emph{fixed-base} systems, such as industrial robot manipulators, whose configuration spaces are vector spaces, computationally efficient algorithms exist to compute the state and input matrices of \eqref{eq:introlinearsystem} \cite{carpentier2018analytical}. 
For \emph{moving-base} systems such as drones, humanoids, and quadrupeds (also known as floating-base systems \cite{featherstone2008rigid}), the configuration space is naturally written as the Cartesian product of the robot \emph{pose} --an element of the Lie group $SE(3)$-- and the robot \emph{shape}, described as $n$ dimensional manifold describing the robot's internal joint displacements. For 1-DOF joints, the most common type of actuated robot joints, the shape manifold can be further thought of as the Cartesian product of $n$ 1-DOF manifolds: one copy of $\mathbb{R}$ for each prismatic joint and one copy of $\mathbb{R}$ or the unit circle $\mathbb{S}^1$ for revolute joints \cite{LeeLeokmcClamroch2017GlobalFormulationLagrangianHamiltonianDynamics}, depending if one needs to treat 360-degrees rotation as being the same configuration or not.
In the robotics literature, revolute joints are commonly parameterized using $\mathbb{R}$ either because robot joints are physically constrained to less than a full rotation or, when a rotation of more than 360 degrees is possible, because encoders allow to count the actual number of rotations: such counting is essential as the robot configuration is actually not the same due to the presence of external cabling needed to collect and send signals from and to the end effector.
This robotic perspective is adopted explicitly in this paper, and therefore we will consider the specific case of moving-base robotic dynamics when the configuration manifold can be parameterized as the Lie group $\SE(3) \times \R^n$. 
In current robotics literature, by means of local parametrization of the orientation (using, e.g., Euler or Tait-Bryan angles), the configuration space of a moving-base system can be artificially considered as that of a fixed-base system, leading to computations performed in the configuration space $\mathbb{R}^{6+n}$, with 
$6$ being the dimension of $\SE(3)$. 
Therefore, algorithms for computing the sensitivity analysis for fixed-based systems can be applied to moving-based systems, although this gives rise to (parametrization) singularity issues \cite{ang1987singularities, diebel2006representing}. 
In this paper, we explore how to obtain the sensitivity of a moving-base system treating the configuration space for what it is, namely, the Lie group $\SE(3) \times \R^n$, without resorting to any local parametrization of the rotation part.

When a system evolves on a Lie group $G$, computing the sensitivity needs a different approach \cite{saccon2013optimal, sonneville2014sensitivity}. 
In \cite{saccon2013optimal}, the authors have provided a definition of sensitivity analysis for systems evolving on a generic Lie group and demonstrated it numerically on the rotational dynamics of a rigid body on $T\SO(3)$. In this approach, that can be described compactly as the linearized  dynamics  written  in  terms of exponential coordinates centered about the nominal trajectory described in a global fashion, the linearized system $\dot z = A z + B w$ evolves on the Lie algebra of the Lie group. 
This theory can applied to full rigid body dynamics (both translational and rotational) defined on $T\SE(3)$ \cite[Chapter 4]{murray2017mathematical}, and even to moving-base dynamics defined on $T(\SE(3) \times \R^n)$ such as those
described in \cite{from2010singularity, pucci2018momentum, ayusawa2008identification}, making it possible to define 
a notion of singularity-free sensitivity for moving-base rigid robot dynamics.
The idea of representing the linearized equations of a dynamic system on a Lie group making use of the tangent space (Lie Algebra) has been presented, limited to equilibrium points, also in \cite[Appendix B]{LeeLeokmcClamroch2017GlobalFormulationLagrangianHamiltonianDynamics}.

Moving-base robots such as humanoids are typically controlled at a discrete samplying time that varies in the range from hundreds to a few KHz and the amount of computations 
for inverse-kinematics-based control
or optimal whole-body control can increase at least quadratically with number of degrees of freedom. Adding to these reactive control 
strategy also predictive or offline planning strategies such as model-predictive control or trajectory optimization, it is
evident that the computational time needed to obtain the sensitivity is a limiting factor and more efficient computational methods are therefore always of interest and of practical value. In the following sections, we provide
an overview of the existing approaches and
highlight our contribution within this context. 

\subsection{Literature overview}

This section provides an overview of existing methods to compute the linearization of the equation of motions of a multibody model of a robotic system.
Four known methods to compute or approximate the sensitivity of multibody systems are discussed in the following subsections: 
\emph{finite differences} \cite{tassa2012synthesis}, 
\emph{Lagrangian derivation} \cite{garofalo2013closed}, 
\emph{automatic differentiation} \cite{giftthaler2017automatic}, and 
\emph{recursive analytical derivation} \cite{carpentier2018analytical}. 
After that, the literature 
about 
\emph{recursive multibody dynamics algorithms} \cite{featherstone2008rigid, luh1980line, walker1982efficient} is reviewed, as these
algorithms are the foundation of 
computationally efficient methods for 
multibody dynamics. 
Finally, sensitivity analysis on Lie groups \cite{sonneville2014sensitivity} and \emph{geometric linearization on Lie groups} \cite{saccon2013optimal, saccon2011lie}, \cite[Section II.G]{sola2018micro} is 
discussed, posing the basis 
for the understanding of this contribution. 

\paragraphbold{Sensitivity analysis using finite differences.} Finite differences is a relatively simple method to approximate the linearization. It evaluates the dynamics several times: once unperturbed, and multiple times with an added perturbation for each degree of freedom and for each input variable. For systems with a large number of degrees of freedom (e.g. humanoids or quadrupeds), this method becomes however time-consuming. Furthermore, the finite difference method is prone to numerical rounding errors. Despite the time-consuming computations, there is literature available showing successful usage of finite differences on real-time applications. In \cite{mason2014full}, a method to compute the sensitivity of moving-base systems is presented. The moving-base is modelled in a singularity-free way, although little details are provided regarding how this is done (sensitivity computations appear to be done at 1 Hz). 
In \cite{tassa2012synthesis}, a method is shown to apply model predictive control to humanoids, using the sensitivity computed by finite differences. The authors of \cite{tassa2012synthesis} claim that almost all CPU time is spent computing the sensitivity. Implementation of this strategy required careful implementation and parallel processing. 

\paragraphbold{Sensitivity analysis using Lagrangian derivation.} 
In \cite{garofalo2013closed}, by taking a Lagrangian perspective,
explicit analytical expressions are provided for the
dynamic matrices and their derivatives. The paper 
provides an interesting historical 
perspective for the need of computing
the sensitivity, however the provided expressions do not use the inherent sparsity of the dynamic matrices, resulting in unnecessary computations. The authors 
of \cite{garofalo2013closed} mention moving-base systems in the introduction, but do not model them in a singularity-free way in the derivation. 

\paragraphbold{Sensitivity analysis using automatic differentiation.} Automatic differentiation (AD) also known as algorithmic differentiation is a software
tool capable of generating 
a new computer program to numerically evaluate the derivative of a function specified by a given computer program. 
At its core,
AD applies the chain rule to all operations and function calls performed by the computer program. More specifically, AD relies on the fact that the derivatives of basic operations (i.e. addition, subtraction, multiplication) and trigonometric functions (e.g., $\sin$, $\cos$, $\exp$) are known to create a new 
binary expression tree that computes 
the derivative starting from the original 
binary expression tree.
In \cite{giftthaler2017automatic, neunert2016fast}, this method is applied to multibody systems. A tool called \textit{RobCoGen} is used to automatically generate robot-specific rigid body dynamics code. Automatic differentiation is applied to this code, to compute the the derivatives. Moving-base systems are mentioned, but the followed approach is to use Euler-like angles to prescribe the orientation and therefore introducing parametric singularity and extra computations deriving from the trigonometric parametrization of the orientation.
In \cite{anderson2002analytical} the author warns that automatic differentiation may give wrong results if it is based on pure syntactical analysis and is implemented without knowledge of the problem structure. 
Generally speaking, AD does not have a notion of geometric differentiation (e.g., the direct derivative of a rotation-matrix-valued function) and therefore it is unclear at this point how a singularity-free geometric derivative could be obtained automatically.

\paragraphbold{Sensitivity analysis using recursive analytical derivation.} Recursive analytical derivation uses recursive algorithms (in particular, the Recursive Newton Euler Algorithm (RNEA) and the Articulated Body Algorithm (ABA)) to obtain the linearization, by line-by-line differentiating these algorithms using the chain rule. The reader not familiar with the algorithms RNEA and ABA  --well-known and established in the robotic community-- is referred to next subsection `Recursive algorithms' for detailed references. Recursive analytical derivation requires algebraic differentiation of spatial algebra. In \cite{sohl2001recursive}, an algorithm which focuses on underactuated systems is presented. This algorithm is a hybrid algorithm: it computes partly forward dynamics for passive (also known as not actuated) joints and partly inverse dynamics for active (actuated) joints. The configuration space is assumed to be a vector space, therefore the moving-base cannot be represented in a singularity-free manner, plus there is the issue of uniqueness of the local parametrization that can avoided only by taking a differential geometry perspective (cf., e.g., \cite{LeeLeokmcClamroch2017GlobalFormulationLagrangianHamiltonianDynamics}). In \cite[Chapter 6]{park2018geometric}, an analytical algorithm computing the sensitivity of the inverse dynamics is presented. Moving-bases are not discussed. The authors of \cite{carpentier2018analytical} present an analytical algorithm that computes the sensitivity of the inverse dynamics, based on the RNEA. The derivatives of the forward dynamics can be computed by deriving the ABA, although the authors state that it may also be found by using a relation between the inverse and forward dynamics, which result in lower computation times. The computation times needed to derive the sensitivity of both the inverse and forward dynamics are found to be much lower compared to the finite difference method. In \cite{carpentier2018analyticalinverse}, the same authors present an algorithm to directly compute the inverse of the mass matrix, without first computing the mass matrix itself. As it is the inverse that is needed, and not the mass matrix itself, this also leads to lower computation times. This method is also found to have lower computational times than the standard approach of using Cholesky decomposition. 

\paragraphbold{Recursive algorithms.} Recursive algorithms are highly efficient methods to compute robot dynamics. The equations of motion of robotic systems
written in matrix form often present sparse matrices, where the sparsity is induced by the tree structure
of robot kinematics. Recursive algorithms exploit the tree-structure-induced matrix sparsity by omitting unnecessary computations. Three well-known algorithms for robotic systems within this context are the Recursive Newton Euler Algorithm (RNEA) \cite{luh1980line}, the Composite Rigid Body Algorithm (CRBA) \cite{walker1982efficient} and the Articulated Body Algorithm (ABA) \cite[Chapter 7]{featherstone2008rigid}. The RNEA computes the inverse dynamics, i.e., the necessary joint torques to achieve desired joint accelerations. The CRBA computes the mass matrix. The ABA computes the forward dynamics, i.e., the resulting joint accelerations with given joint torque. In \cite{featherstone2008rigid}, an overview of these three algorithms is presented, as well as how to use these algorithms obtain 
the equations of motion for multibody systems. In  \cite[Chapter 9]{featherstone2008rigid}, 
floating-base versions of the recursive algorithms
are presented. These modification are fully geometric, i.e., they make no use of Euler-like parameterization for the orientation as well as the angular velocities and accelerations. We will refer to them in the following as the \emph{Recursive Newton Euler Algorithm for moving-base systems} (RNEAmb), the \emph{Composite Rigid Body Algorithm for moving-base systems} (CRBAmb), and the \emph{Articulated Body Algorithm for moving-base systems} (ABAmb). 

\paragraphbold{Sensitivity on Lie groups and
the geometric linearization.} The theory of geometric linearization for a continuous-time dynamical system evolving on an arbitrary matrix Lie group  has been developed within
the context of geometric numerical optimal 
control in \cite{Saccon2010ProntoLieCDC}
and further refined and illustrated in 
\cite{saccon2013optimal, saccon2011lie}.
It allows to compute the sensitivity for control systems on Lie groups, by exploiting the fact that a Lie group's Lie algebra is a vector space with a special binary operation. In \cite{saccon2013optimal}, the authors apply this method to linearize rotational dynamics 
on $T \SO(3)$, demonstrated in the context of numerical optimal control.  The resulting linearization does not suffer from  parameterization-induced singularities
and the obtained expressions can be straightforwardly validated against a brute-force (geometric) finite difference approach.

In the context of multibody systems, a strictly related and independently developed approach for computing the continuous-time geometric sensitivity is presented in \cite{sonneville2014sensitivity},
stemming from \cite{Bruls2008linearization}. The approach in \cite{sonneville2014sensitivity} does not address the exploitation of sparsity by means of a recursive formulation to speed up the computations. 
The same authors
are however well aware that the exploitation of sparsity/recursion would be computationally advantageous, as validated in a specific example in
\cite{Docquier2019multibody}.

In \cite[Section II.G]{sola2018micro}, a compact introduction to Lie groups theory and its use in robotics for state estimation is given, which too shows the same concept of continuous-time geometric linearization (apparently being unaware of previously published results in this direction). A software implementation of these concepts is presented in \cite{Deray12020Manif}.

Although in this work we focus on the (recursive) computation continuous-time linearization of a mechanical system evolving on a Lie group, for sake of completeness, in the following we provide an overview of geometric linearization concepts that has been developed for discrete-time systems on Lie groups.
In the context of discrete-time variational integrators and discrete optimal control on Lie groups, discrete-time linearized equations of motion for the specific case of $SE(3)$ are presented in \cite{leok2007overview} and \cite{lee2008computational}. In the context of discrete-time Bayesian Filtering on Lie groups, a linearization of Lie group variational integrator have been developed in \cite{Sanyal2008} in order to propagate the estimates and the uncertainty ellipsoids of an Extended Kalman Filter. A related work about filtering on Lie groups, that develops a linearization of the invariant error around a fixed point, is presented in \cite{barrau2014intrinsic} and used
to construct a Kalman filter for the linearized model. Finally, in \cite{murphey2015}, a discrete-time linearization for discrete-time dynamics on Lie groups is presented. While it is expected that the various concepts of discrete-time linearization developed in these work are consistent with the continuous-time version employed here as time step goes to zero, such an analysis is beyond the scope of the present work.

\subsection{Contribution}

The main contribution of the present manuscript is the development of a recursive algorithm to obtain the singularity-free geometric linearization of moving-base multibody robotic systems. 
In the development of this algorithm, the following four requirements are explicitly addressed:
\begin{itemize}
  \item[~\textbf{(MB)}] Computation of the sensitivity for rigid \textbf{m}oving-\textbf{b}ase multibody dynamics
  \item[~\textbf{(SF)}] \textbf{S}ingularity-\textbf{f}ree representation of (the orientation of) the moving base
  \item[~\textbf{(RF)}] \textbf{R}ecursive \textbf{f}ormulation, exploiting
  the tree-structure kinematics
  \item[~\textbf{(ED)}] \textbf{E}xact \textbf{d}erivatives (no approximations as, e.g., finite differences)
\end{itemize}
To the best of our knowledge, there is no existing approach which satisfies all these four requirements.
The expectation is that such a recursive formulation, as well known in non-geometric context \cite{carpentier2018analytical, featherstone2008rigid}, will lead to a substantial computational speed up. This detailed computation efficiency analysis is however deemed as future work, the main focus of this work being the detailing and verification of the correctness of the derived algorithm.
A visual comparison, highlighting the similarity/difference with the existing approaches, is provided by Table \ref{tab:research}.

\begin{table}[h]
  \centering
  \caption{Overview of existing approaches and 
  their compliance with the four requirements 
  (MB, SF, RF, ED).}
  \label{tab:overview}
  \begin{tabular}{ | l | c | c | c | c | }
    \hline
    \textbf{Approach} 
      & \textbf{SF} 
      & \textbf{MB} 
      & \textbf{RF}
      & \textbf{ED}
      \\ \hline
    Geometric linearization \cite{Sanyal2008, barrau2014intrinsic, saccon2013optimal, saccon2011lie, murphey2015, leok2007overview, lee2008computational, sola2018micro}
      & \checkmark 
      & 
      & 
      & \checkmark
    \\ \hline  
    Sensitivity
    for multibody systems 
    on Lie groups
    \cite{sonneville2014sensitivity, Docquier2019multibody}
      & \checkmark 
      & \checkmark 
      & 
      & \checkmark
      \\ \hline
    Recursive algorithms \cite{featherstone2008rigid, luh1980line, walker1982efficient} 
      & \checkmark
      &
      & \checkmark 
      & \checkmark
      \\ \hline
    Finite differences \cite{tassa2012synthesis, mason2014full}
      & \checkmark
      & \checkmark
      & 
      & 
      \\ \hline
    Lagrangian derivation \cite{garofalo2013closed}
      & 
      & \checkmark
      & 
      & \checkmark
      \\ \hline
    Automatic differentiation \cite{giftthaler2017automatic, neunert2016fast}
      &
      & \checkmark
      & \checkmark
      & \checkmark
      \\ \hline
    Analytical derivation  \cite{carpentier2018analytical, sohl2001recursive, park2018geometric} 
      & 
      & \checkmark 
      & \checkmark 
      & \checkmark
      \\ \hline
    This manuscript  
      & \checkmark 
      & \checkmark 
      & \checkmark 
      & \checkmark
      \\ \hline
    \end{tabular}
  \label{tab:research}
\end{table}

Without loss of generality, the derivation
presented in this document will make use of the following assumption:
\begin{assumption} \label{assumption:joints}
We assume that all joints but the moving base are \textit{conventional} joints, i.e., joints of the following types: revolute, prismatic, helical, cylindrical or planar.
\end{assumption}

Furthermore, we limit our investigation to the case of unconstrained systems:
\begin{assumption} \label{assumption:externalforces}
We assume that the system is unconstrained (i.e.,
that are no closed kinematics loops/contacts and resulting contact/constraint forces acting on the system).
\end{assumption}
More specifically, a thorough discussion about the  geometric constrained linearization is left 
for future work, knowing that the geometric unconstrained linearization can be efficiently exploited for computing the constrained linearization as done in commercial multibody software (see, e.g.,  \cite{Negrut2006}).

The present contribution has been achieved by setting and reaching the following three objectives:
\begin{enumerate}
  \item Derive and present the mathematical formulas for the singularity-free geometric linearization of moving-base multibody systems. 
  \item Derive numerically efficient and accurate algorithms to compute the singularity-free geometric linearization of moving-base multibody systems.
  \item Verify the correctness of the derived algorithms.
\end{enumerate}
This manuscript details these 
accomplishments, together with providing necessary mathematical background, as detailed in the following section.

\subsection{Outline} 

In addition to this introduction,
the present manuscript 
is structured as follows. 
Section \ref{chp:preliminaries} introduces notation, multibody system definitions, the equations of motion for moving-base systems, and 
provides the basis of
(left-trivialized) geometric linearization.
Section \ref{chp:theoreticalaspects} presents the explicit expressions for the geometric linearization for moving-base multibody systems. 
Section \ref{chp:algorithmicaspects} shows how the
geometric linearization
can be computed 
efficiently by means of newly derived numerical recursive  algorithms. Numerical validation of the computationally efficient 
algorithm for geometric linearization of  
moving-based multibody systems
is presented in Section \ref{chp:numericalverification}.
Conclusion and future research directions are presented in Section \ref{chp:conclusions}.

\section{Preliminaries}
\label{chp:preliminaries}

In this section we introduce the notation, definitions, 
and equations of motion for moving-base systems
employed in what follows, and we recall the theory of left-trivialized 
geometric linearization.
We assume the reader to
be have some familiarity with the basic concepts of differential geometry \cite{lee2013smooth} and,
more in particular, 
matrix Lie groups \cite{varadarajan2013lie, rossmann2002lie}.

\subsection{Lie groups and differential geometry related notation} 
\hspace*{\fill} 

\renewcommand{\arraystretch}{1.0}
~\\
\begin{tabular}{ll}
$M$, $N$ & Smooth manifolds \\
$x$ & Point on a manifold \\
$T_x M$, $T_x^* M$ & Tangent and cotangent spaces of $M$ at $x$ \\
$TM$, $T^*M$ & Tangent and cotangent bundles of $M$ at $x$ \\
$f : M \rightarrow N$ & (Smooth) mapping for $M$ to $N$ \\
$\D f : TM \rightarrow TN$ & Tangent map of $f$ \\
$\mathrm{D}_i f : TM \rightarrow TN$ & Tangent map of $f$ with respect to the $i$-th argument of \\
    & $f: \dots \times M \times \dots  \rightarrow N$ \\
$G$ & Lie group \\
$g \in G$ & Element of Lie group \\
$e$ & Group identity \\
$\cdot_G$ & Operation associated to the Lie group $G$ \\
$\mathfrak{g}$ & Lie algebra of $G$ \\
$[\cdot , \cdot]_\mathfrak{g}$ & Lie brackets on $\mathfrak{g}$ \\
$L_g x$, $R_g x$ & Left and right translations of $x \in G$ by $g \in G$ \\
$gx$ , $xg$ & Shorthand notation for $L_g x$, $R_g x$ \\
$gv$ , $vg$ & Shorthand notation for $\D L_g(x)\cdot v$, $\D R_g(x)\cdot v$ 
              with $v \in T_x G$ \\
$\Ad$ & Adjoint representation of a Lie group to its algebra \\
$\ad$ & Adjoint representation of a Lie algebra onto itself \\
$S_1 \times S_2$ & Cartesian product of sets $S_1$ and $S_2$ \\
$G_1 \times G_2$ & Direct product of the Lie groups $G_1$ and $G_2$ \\
$\mathfrak{g}_1 \oplus \mathfrak{g}_2$ & Direct sum of the Lie 
              algebras $\mathfrak{g}_1$ and $\mathfrak{g}_2$ \\
$\exp : \mathfrak{g} \rightarrow G$ & Exponential map of $G$ \\
$\log : G \rightarrow \mathfrak{g}$ & Logarithm map (inverse 
    of exp in a neighbourhood of $e$) \\
$\SO(3)$ & Special Orthogonal group of dimension 3 \\
$\SE(3)$ & Special Euclidean group of dimension 3 \\
$\so(3)$ & Lie algebra of $\SO(3)$ \\
$\se(3)$ & Lie algebra of $\SE(3)$ \\
$\R^3_\times$ & Lie algebra given by $\R^3$ with the cross product as 
                Lie bracket \\ 
$\R^6_\times$ & Lie algebra given by $\R^6$ with the 6D cross product 
                as Lie bracket \\ 
$\wedge \text{~(read:~wedge)}$ & Lie algebra isomorphism from $\R^6_\times$ to $\se(3)$ 
                   (or from $\R^3_\times$ to $\so(3)$) \\
$\vee \text{~(read:~vee)}$ & Lie algebra isomorphism from $\se(3)$ to $\R^6_\times$ (or 
    from $\so(3)$ to $\R^3_\times$) \\
$I_n \in \R^{n \times n}$ & Identity matrix of dimension $n$ \\
$\partial x / \partial y$ & Partial derivative of $x$ with respect to 
    scalar or vector $y$ \\    
$\tilde{\partial} x / \partial y$ & Left-trivialized partial derivative of $x$ with respect to Lie \\
    & group element $y$ (defined in Section \ref{sec:diffEIDA}) \\
$[a;b]$ & row concatenation operator ($[a;b] := [a^T,b^T]^T$) \\
\end{tabular}

\subsection{Definitions for moving-base systems}
\label{subsec:definitionsformovingbasesystems}
In this subsection, all definitions regarding bodies and joints are presented. These are used, in particular, in the multibody algorithms presented in Section \ref{chp:algorithmicaspects}. 
First of all, we define a multibody system as a group of rigid bodies connected by joints, so that the bodies may undergo relative translational and rotational displacements.

It is possible to model multiple-degrees-of-freedom 
joints (namely, cylindrical and planar) as multiple 1-DoF joints (namely, revolute, prismatic, helical), which allows simplification of the recursive algorithms
(in this manuscript, we just present the case of 1-DoF joints). The cylindrical joint can be modelled as a combination of a revolute and a prismatic joint about the same axis. The planar joint can be modelled as a combination of two prismatic joints in the same plane and a rotational joint orthogonal to it. 
We warn the reader that rewriting joints as a cascade of 1-DoF joints is not fully possible for a spherical joint (also known as ball and socket joint). A spherical joint's configuration is described by a rotational matrix in $\SO(3)$. They
should not be confused with ideal 2-DoF joints evolving on the unit sphere $\mathbb{S}^2$ and whose geometric dynamics is described, e.g., in the excellent monograph \cite[Chapter 5]{LeeLeokmcClamroch2017GlobalFormulationLagrangianHamiltonianDynamics}.
For sake of ease of presentation, we have decided to leave out a detailed discussion about how to handle spherical joints both from a modeling and geometric linearization perspective. 
It should be noted, however, that because a spherical joint's displacement can be represented as a rotational matrix, the developed machinery of geometric linearization, detailed in this manuscript, 
would allow to handle them directly in a geometrically
consistent and computationally efficient manner.

Body numbering starts at the moving-base, which is defined as body $0$, and going outwards the body numbers increase up to $n_B$, the \emph{total number of bodies excluding the moving base}, implying that each body's number must be higher than the one of its parent (there are clearly multiple possibilities to number tree-based multibody systems). 
The parent of body $i$ will be indicated with $\lambda(i)$ and its children with $\mu(i)$, that is a set of indices (array-structure). An example of body numbering is depicted in Figure \ref{fig:body_numbering}, where $\mu(i) = \{j,k\}$. 

\begin{figure}[h]
  \centering
  \includegraphics[scale=0.8]{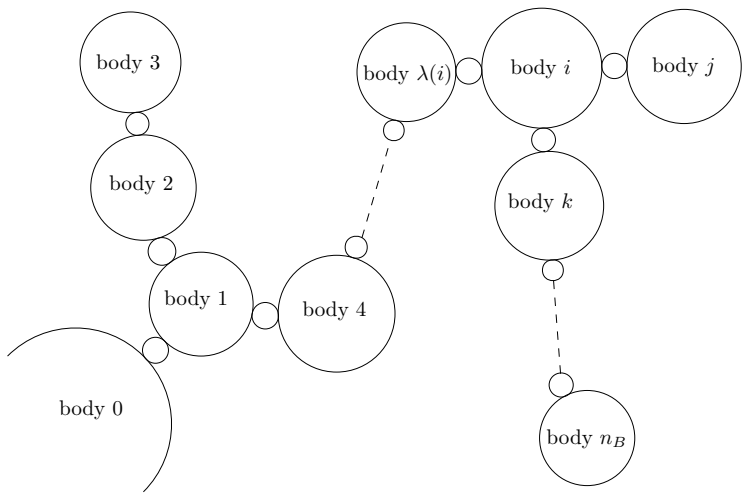}
  \caption{An example of body numbering for moving-base multibody systems with tree-topology kinematics. The moving base is denoted $0$. In the image, $j > i$ and $k > i$. The parent body of body $i$ is denoted $\lambda(i)$.}
  \label{fig:body_numbering}
\end{figure}

By convention, we define joint $i$ to be the joint connecting body $i$ (successor) to body $\lambda(i)$ (predecessor), as shown in Figure \ref{fig:joint_numbering}. Note that the number of joints $n_J$ is equal to the number of bodies, $n_B$.

\begin{figure}[ht]
  \centering
  \includegraphics[scale=0.8]{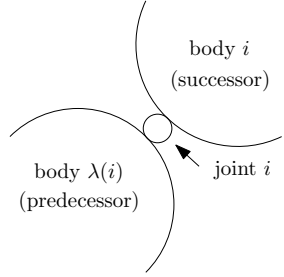}
  \caption{Joint numbering convention. Joints are numbered according to the successor body.}
  \label{fig:joint_numbering}
\end{figure}

Various coordinate frames are defined to 
express the position of rigid bodies in space. In particular, frame $A$ represents the \emph{inertial frame}. 
Frame $i$ represents the frame attached of body $i$ at the location where it connects to its predecessor joint $\lambda(i)$. 
Frame $\lambda(i)|i$ represents the frame attached to body $\lambda(i)$ at the location where it connects to joint $i$. 
The transformation matrix from frame $\lambda(i)$ to frame $\lambda(i)|i$ is denoted  $\ls^{\lambda(i)|i}\vH_{\lambda(i)} \in \SE(3)$ and 
it is a constant parameter of body $\lambda(i)$ because both frames are fixed relative to body $\lambda(i)$. 
The transformation matrix from frame $\lambda(i)|i$ to frame $i$ is denoted $\ls^i\vH_{\lambda(i)|i}$ (it is a parameterized rotation about the joint axis if, e.g.,
the joint is revolute). An example is shown in Figure \ref{fig:framesi}. 

\begin{figure}[h]
  \centering
    \includegraphics[scale=0.8]{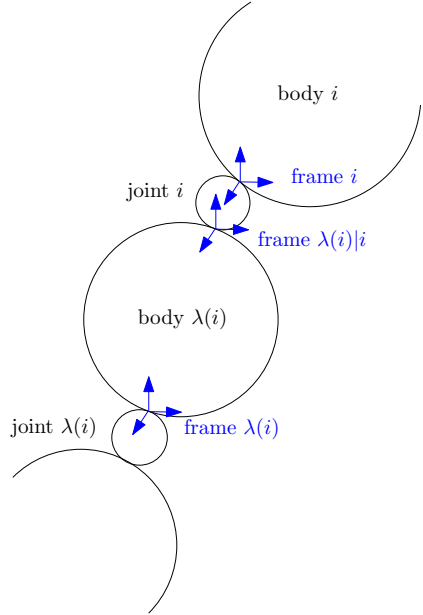}
    \caption{A general example of body and joint frames.}
    \label{fig:framesi}
\end{figure}

To represent quantity used in recursive multibody dynamics algorithms, in this manuscript use Eindhoven-Genoa (EG) notation~\cite{traversaro2019multibodyv2}. This notation is inspired by the one used in multibody dynamics algorithms literature~\cite{featherstone2008rigid}, but clearly defined to be compact, non-ambiguous, and in harmony with Lie Group formalism. A short introduction of this notation is given in Table \ref{tab:EGnotation}. For more details about the notation, see \cite{traversaro2019multibodyv2}.

\begingroup
\renewcommand{\arraystretch}{1.2}
\begin{table}[H]
\caption{Introduction of Eindhoven-Genoa (EG) notation.} 
\label{tab:EGnotation}
\centering
\begin{tabular}{|p{1.3cm}|p{1.8cm}|p{8.5cm}|}
\hline
\textbf{EG} & \textbf{Dimension} & \textbf{Explanation} \\ \hline
$\ls^A\vH_B$
   & $\SE(3)$
   & Transformation matrix of frame $B$ w.r.t. frame $A$
   \\ \hline
$\ls^A\vR_B$
   & $\SO(3)$
   & Rotation matrix of frame $B$ w.r.t. frame $A$
   \\ \hline
$\ls^A\vo_B$
   & $\R^3$
   & Origin of frame $B$ w.r.t. frame $A$
   \\ \hline
$\ls^C\vv_{A,B}$
   & $\R^6$
   & Twist of frame $B$ w.r.t. frame $A$ expressed in frame $C$
   \\ \hline
$\ls^A\va_{A,B}$
   & $\R^6$
   & Intrinsic \cite[Section 5.1]{traversaro2019multibodyv2} acceleration of frame $B$ w.r.t. frame $A$ expressed in frame $C$
   \\ \hline
$\ls_A\vf$
   & $\R^6$
   & Wrench w.r.t. frame $A$ (often written as $\vb$ for bias wrench)
   \\ \hline
$\ls^A\vX_B$ 
   & $\R^{6 \times 6}$
   & Velocity transformation of frame $B$ w.r.t. frame $A$
   \\ \hline
$\ls_A\vX^B$ 
   & $\R^{6 \times 6}$
   & Wrench transformation of frame $B$ w.r.t. frame $A$
   \\ \hline
$\vs$ 
   & $\R^{n_J}$
   & Generalized position vector or system shape
   \\ \hline
$\vr$ 
   & $\R^{n_J}$
   & Generalized velocity vector
   \\ \hline
$\vtau$ 
   & $\R^{n_J}$
   & Joint torques or generalized forces vector
  \\ \hline
$\ls^C\vv_{A,B} \times$
   & $\R^{6 \times 6}$
   & 6D twist cross product on $\R^6$ (defined in Section \ref{sec:Liegroupofmovingbasesystems}) 
   \\ \hline
$\ls^C\vv_{A,B} \bar{\times}^*$
   & $\R^{6 \times 6}$
   & 6D twist/wrench cross product on $\R^6$ 
   \\   
   \hline
\end{tabular}
\end{table}
\endgroup

\subsection{Equations of motion for moving-base systems}
\label{subsec:equationsofmotion}
We write the configuration of a moving-base system with $n_J$ 1-DoF joints as $\vq := (\vH,\vs)$, where $\vH := \ls^A\vH_0 \in \SE(3)$ is the moving-base transformation matrix and $\vs \in \R^{n_J}$ the joint displacements vector. The total number of degrees of freedom is denoted by $n = n_J + 6$. 
The time-derivative of the configuration is given by $\dot{\vq} = (\dot{\vH},\dot{\vs}) \in T\SE(3)\times T\R^{n_J}$, which we identify with $T\SE(3)\times\R^{n_J}$. Through left-trivialization,
the system velocity 
$\dot \vq$ will be written equivalently as $\vnu := (\vv,\vr)$, where $\vv := \ls^0\vv_{A,0} \in \R^6$ is the moving-base twist and $\vr := \dot{\vs} \in \R^{n_J}$ is the generalized velocity vector. More
explicitly, the kinematics of the moving base and of the joints are given by
\begin{align}
    \dot{\vH} &= \vH\vv^\wedge , \label{eq:kinbase} \\
    \dot{\vs} &= \vr . \label{eq:kinjoints}
\end{align}

The unconstrained forced dynamics of a moving-base system in $\SE(3) \times \R^{n_J}$ is 
\begin{equation} \label{eq:dynamicsgeneral}
	\vM \dot{\vnu} + \vC \vnu + \vG = \vS \vtau ,
\end{equation}
where $\vM \in \R^{n \times n}$ is the mass matrix, $\vC \in \R^{n \times n}$ the Coriolis matrix, $\vG \in \R^n$ the gravitational wrench vector, $\vS := [0_{6 \times n_J} ; I_{n_J}] \in \R^{n \times n_J}$ the joint selection matrix, and $\vtau \in \R^{n_J}$ the joint torques. 
Employing the generalized bias vector 
$\vh := \vC \vnu + \vG \in \R^n$, \eqref{eq:dynamicsgeneral} 
can be rewritten as
\begin{equation}
    \vM \dot{\vnu} + \vh = \vS \vtau . \label{eq:dynamicsbias}
\end{equation}
In matrix form, \eqref{eq:dynamicsbias} reads
\begin{equation}
    \begin{bmatrix} 
        \vM_{11} & \vM_{12} \\
        \vM_{21} & \vM_{22}
    \end{bmatrix}
    \begin{bmatrix}
        \dot{\vv} \\
        \dot{\vr}
    \end{bmatrix} 
    +
    \begin{bmatrix}
        \vh_1 \\
        \vh_2
    \end{bmatrix}
    =
    \begin{bmatrix} 
        0_{6 \times 1} \\
        \vtau
    \end{bmatrix} ,
\end{equation}
where $\vM_{11} \in \R^{6 \times 6}$ is 
sometimes called the locked inertia 
\cite{saccon2017centroidal} corresponding
to the 6D inertia matrix of the whole system when considering its joints as locked (composite rigid body), $\vM_{12} = \vM_{21}^\mathrm{T} \in \R^{6 \times n_J}$ is a matrix stacking the wrenches required to support unit acceleration \cite[Section 9.3]{featherstone2008rigid}, $\vM_{22} \in \R^{n_J \times n_J}$ the generalized inertia matrix
corresponding to assuming the moving base as fixed, $\vh_1 \in \R^6$ the bias wrench for the whole system as a composite rigid body and $\vh_2 \in \R^{n_J}$ the generalized bias wrench vector of the joints.

\subsection{Left-trivialized geometric linearization on Lie groups} 
In this section, we recall the definition of left-trivialized geometric linearization as detailed in \cite{saccon2013optimal, saccon2011lie}, which is the cornerstone of our derivation.

Given a Lie group $G$, a controlled dynamical system evolving on $G$ is defined as
\begin{equation} \label{eq:dynsysgeneral}
    \dot{g}(t) = f(g,u,t),
\end{equation}
where $g \in G$ is the system state, $u \in \R^m$ the system input and $t \in \R$ is time. 
The left-trivialized vector field associated to
$f$ is the map $\lambda : G \times \R^m \times \R \rightarrow \mathfrak{g}$, 
\begin{equation} \label{eq:lefttrivgeneral}
    \lambda(g,u,t) := g^{-1}(t) f(g,u,t),
\end{equation}
that allows to rewrite \eqref{eq:dynsysgeneral} as
\begin{equation}
    \dot{g}(t) = g(t) \lambda(g,u,t).
\end{equation}
Given a nominal trajectory $\eta(t) := (g(t),u(t)) \in G \times \R^m$, $t \geq 0$, the left-trivialized geometric linearization of the dynamical system \eqref{eq:dynsysgeneral} about $\eta(t)$ is the time-varying linear system 
\begin{equation} \label{eq:lefttrivlin}
    \dot{z}(t) = A(\eta,t) z(t) + B(\eta,t) w(t),
\end{equation}
with $z(t) \in \mathfrak{g}$ the perturbation vector, $w(t) \in \R^m$ the perturbed input vector,  
\begin{align}
    \label{eq:matrixAgeneral}
    A(\eta,t) &:= \D_1 \lambda(g,u,t) \circ \D L_g (e) - \ad_{\lambda(g,u,t)} 
    \shortintertext{and}
    \label{eq:matrixBgeneral}
    B(\eta,t) &:= \D_2 \lambda(g,u,t),
\end{align}
called respectively the state and input linearization matrices.
In the following section, we will show how to interpret
\eqref{eq:dynamicsgeneral} as a controlled dynamical system 
on a suitable Lie group and then compute 
explicitly its left-trivialized
linearization.

\section{Left-Trivialized Linearization for Multibody Dynamics}
\label{chp:theoreticalaspects}

In this section, we provide the unconstrained moving-base multibody dynamics with a Lie group structure and use this Lie group
to derive the state and input linearization matrices
associated to the multibody dynamics.
Furthermore, we show how to rewrite this linearization, which is naturally written in terms of the forward dynamics, by using inverse dynamics. This later step is a necessary bridge to the computationally efficient implementation of the geometric linearization making use of recursive algorithms that will be presented in Section~\ref{chp:algorithmicaspects}.

\subsection{Lie group of moving-base systems}
\label{sec:Liegroupofmovingbasesystems}

Under the assumption, typically encountered in practice, that revolute joints configuration
include the information about number of revolutions (cf. Section~\ref{sec:motivation}), the configuration manifold of a moving-base robot composed of a moving base to which $n_J$ 1-DoF joints are attached is given by
\begin{equation}
    Q = \SE(3) \times \R^{n_J} .
\end{equation}

A generic configuration $\vq$ is then written as $\vq = (\vH,\vs) \in Q$, as presented in Section~\ref{subsec:equationsofmotion}.
The state manifold of the moving-base system is
\begin{equation}
    TQ = T \big( \SE(3) \times \R^{n_J} \big) ,
\end{equation} 
which can be identified (by means of a diffeomorphism) with 
\begin{equation} \label{eq:statespace}
    TQ \cong \SE(3) \times \R^{n_J} \times \R^6 \times \R^{n_J}
\end{equation}
by applying left-trivialization.
An element of the state manifold $TQ$ can be therefore written as $(\vH, \vs, \vv, \vr) \in TQ$, where $\vv$ and $\vr$ represent the velocity of the moving base and that of the joints, respectively. 

The state manifold \eqref{eq:operation} is not a Lie group, as we have not yet defined an operation. Taking two elements
$(\vH_1, \vs_1, \vv_1, \vr_1)$ and $(\vH_2, \vs_2, \vv_2, \vr_2)$ of (the set) $\SE(3) \times \R^{n_J} \times \R^6 \times \R^{n_J}$, we define the group operation as 
\begin{equation} \label{eq:operation}
    (\vH_1, \vs_1, \vv_1, \vr_1) \cdot (\vH_2, \vs_2, \vv_2, \vr_2) = (\vH_1 \vH_2, \vs_1 + \vs_2, \vv_1 + \vv_2, \vr_1 + \vr_2) .
\end{equation}
It is straightforward to verify that \eqref{eq:statespace} equipped with \eqref{eq:operation} forms a Lie group. Note that there exist alternative group operations which could have be chosen to turn \eqref{eq:statespace}
into a Lie group (one being the so called \emph{tangent group}).  In \cite{saccon2013optimal}, it has been shown that (for $\SO(3)$) a Lie group operation defined as in \eqref{eq:operation} leads to simpler computational expressions for the geometric linearization and therefore such an hindsight is also exploited here. Therefore, we defined the Lie group $G$ of moving-base systems as\footnote{$\times$ denotes the group direct product.} $\SE(3) \times \R^{n_J} \times \R^6 \times \R^{n_J}$, by combining the set \eqref{eq:statespace} with the operation \eqref{eq:operation} .

The Lie algebra of $G$ is\footnote{$\oplus$ denotes the direct sum of two Lie algebras} $
    \mathfrak{g} 
    = 
	\se(3) \oplus \R^{n_J} \oplus \R^6 \oplus \R^{n_J}$, 
which we can identify with $
	\mathfrak{g} 
	\cong
	\R^6_\times \oplus \R^{n_J} \oplus \R^6 \oplus \R^{n_J}$. 
We recall that the sub-algebra $\R^6_\times$ is the vector space $\mathbb{R}^6$ with the Lie bracket given by the 6D cross product defined by
\begin{equation}
\label{eq:6dcrossproduct}
  \vv \times 
  := 
  \begin{bmatrix}
    \bmomega^\wedge & 
    \bmv^\wedge \\
    0_{3 \times 3} & 
    \bmomega^\wedge
  \end{bmatrix} 
  , 
  \quad \textrm{where} \quad 
  \vv 
  = 
  \begin{bmatrix}
    \bmv \\
    \bmomega 
  \end{bmatrix} ,
\end{equation}
with $\bmv \in \R^3$ and $\bmomega \in \R^3$ (representing the linear and angular velocity of the moving base, respectively).

\subsection{State and input linearization matrices for moving-base systems}
Recalling that $G = \SE(3) \times \R^{n_J} \times \R^6 \times \R^{n_J}$,
from \eqref{eq:dynamicsbias}, we can write the forward dynamics $FD : G \times \R^{n_J} \rightarrow \R^n$ relating configuration, velocity, and input to acceleration (i.e., the time derivative of the velocity $\vnu := (\vv,\vr)$) as
\begin{equation} \label{eq:forwarddynamics}
	FD(\vH,\vs,\vv,\vr,\vtau) 
	:=
	\dot{\vnu}(\vH,\vs,\vv,\vr,\vtau) 
	=
	\vM^{-1}(\vs) \big[ - \vh(\vH,\vs,\vv,\vr) + \vS \vtau \big] .
\end{equation}
The forward dynamics \eqref{eq:forwarddynamics} can be then partitioned as $FD = [FD_b ; FD_j]$ with
\begin{align} 
    FD_b (\vH,\vs,\vv,\vr,\vtau) 
    &:=
    \dot{\vv}(\vH,\vs,\vv,\vr,\vtau) 
    \label{eq:forwarddynamicsbase} ,
    \\
    FD_j (\vH,\vs,\vv,\vr,\vtau) 
    &:=
    \dot{\vr}(\vH,\vs,\vv,\vr,\vtau) 
    \label{eq:forwarddynamicsjoints} ,
\end{align}
with $FD_b : G \times \R^{n_J} \rightarrow \R^6$ the forward dynamics of the moving base and $FD_j : G \times \R^{n_J} \rightarrow \R^{n_J}$ the forward dynamics of the joints. From the forward dynamics, we obtain the associated left-trivialized linearization, as detailed in the following proposition.

\begin{proposition}
    \label{proposition:linearization}
    The left-trivialized linearization of the moving-base kinematics $\dot{\vH} = \vH \vv$ and $\dot{\vs} = \vr$ and the forward dynamics \eqref{eq:forwarddynamics} about a given trajectory $\eta(t) = \left(g(t),u(t)\right) \in G \times \R^{n_J}$, $t \in \R$, is given by $\dot{z}(t) =  A(\eta,t) z(t) + B(\eta,t) w(t)$, where $z = [z_H; z_s; z_\rmv; z_r] \in \mathfrak{g}$, is the perturbation vector, $w = \vtau \in \R^{n_J}$ 
    the  perturbed input vector, and where the state matrix $A$ and input matrix $B$ are given by\footnote{number of rows and columns of each submatrices of $A$ and $B$ are explicitly indicated in the expressions to indicate the size of each block.}
    \begin{align} \label{eq:matrixA}
        A(\eta,t) 
        &= 
        \begin{blockarray}{ccccc}
          6 & n_J & 6 & n_J & \vspace{0.2cm} \\
          \begin{block}{[cccc]c}
            - \vv \times & 0 & I & 0 & \hspace{0.3cm} 6 \\
                         0 & 0 & 0 & I & \hspace{0.3cm} n_J \\
            \D_1 FD_b \circ \D L_H (I) 
            & \D_2 FD_b 
            & \D_3 FD_b 
            & \D_4 FD_b
            & \hspace{0.3cm} 6 \\
            \D_1 FD_j \circ \D L_H (I) 
            & \D_2 FD_j 
            & \D_3 FD_j 
            & \D_4 FD_j 
            & \hspace{0.3cm} n_J \\
          \end{block}
        \end{blockarray} ,
    \shortintertext{and}
    \label{eq:matrixB}
        B(\eta,t) 
        &= 
        \begin{blockarray}{cc}
          n_J \vspace{0.2cm} & \\
          \begin{block}{[c]c}
            0         & \hspace{0.3cm} 6  \\
            0         & \hspace{0.3cm} n_J \\   
            \D_5 FD_b & \hspace{0.3cm} 6 \\
            \D_5 FD_j & \hspace{0.3cm} n_J \\
          \end{block} 
        \end{blockarray} ,
    \end{align}
    with the derivatives being evaluated at $FD_b( \vH, \vs, \vv, \vr, \vtau )$ and $FD_j( \vH, \vs, \vv, \vr, \vtau )$. 
    $\hfill \blacksquare$
\end{proposition}

\noindent
\textit{Proof:} Using $g = (\vH,\vs,\vv,\vr)$ and $u = \vtau$, the left-trivialized vector field $\lambda (g,u,t)$ of \eqref{eq:lefttrivgeneral} for the dynamics of comprising \eqref{eq:kinbase}, \eqref{eq:kinjoints}, and \eqref{eq:forwarddynamics} is 
\begin{align}
	\lambda (g,u,t) 
	= (\vv, \vr, \dot{\vv}, \dot{\vr})
	\in \mathfrak{g}
	\label{eq:lambda} .
\end{align}
For the Lie group $G = \SE(3) \times \R^{n_J} \times \R^6 \times \R^{n_J}$, the state matrix $A$ in \eqref{eq:matrixAgeneral} and input matrix $B$ in \eqref{eq:matrixBgeneral} need to be compute as
\begin{align}
    A(\eta,t) 
    &:=
    \D_1 \lambda (g,u,t) \circ \D L_{( \vH, \vs, \vv, \vr)} (I, 0, 0, 0) - \ad_{\lambda (g,u,t)} 
    \shortintertext{and}
    B(\eta,t) 
    &:=
    \D_2 \lambda (g,u,t) ,
\end{align}
with $(I, 0, 0, 0)$ the identity of $G$.
The matrix form of the adjoint representation
of the Lie algebra onto itself
evaluated at $\lambda(g,u,t)$, namely $\mathrm{ad}_{\lambda(g,u,t)}$, is given by 
\begin{equation} \label{eq:adjointA}
    \ad_{\lambda (g,u,t)} 
    = 
    \begin{blockarray}{ccccc}
      6 & n_J & 6 & n_J & \vspace{0.2cm} \\
      \begin{block}{[cccc]c}
        \vv \times & 0 & 0 & 0 & \hspace{0.3cm} 6   \\
        0          & 0 & 0 & 0 & \hspace{0.3cm} n_J \\
        0          & 0 & 0 & 0 & \hspace{0.3cm} 6   \\
        0          & 0 & 0 & 0 & \hspace{0.3cm} n_J \\
      \end{block}
    \end{blockarray}.
\end{equation}
The left-translated tangent map $\D L_{( \vH, \vs, \vv, \vr)} (I, 0, 0, 0)$ evaluated in the direction $z = (z_H, z_s, z_\rmv, z_r)$ is instead given by
\begin{equation} \label{eq:lefttrivtangentmap}
    \D L_{( \vH, \vs, \vv, \vr)} (I, 0, 0, 0) \cdot z 
    =
    (\D L_H (I) \cdot z_H , z_s , z_\rmv , z_r) .
\end{equation}
The left-trivialized tangent map $ \D_1 \lambda (g,u,t) \circ \D L_{( \vH, \vs, \vv, \vr)} (I, 0, 0, 0)$ is found by combining the forward dynamics of the base and joints \eqref{eq:forwarddynamicsbase} and \eqref{eq:forwarddynamicsjoints}, the left-trivialized vector field $\lambda(g,u,t)$ \eqref{eq:lambda} and the left-translated tangent map \eqref{eq:lefttrivtangentmap}, resulting in
\begin{align}
     \D_1 \lambda (g,u,t) &\circ \D L_{( \vH, \vs, \vv, \vr)} (I, 0, 0, 0) 
     =
     \nonumber
     \\
     & \begin{blockarray}{ccccc}
      6 & n_J & 6 & n_J & \vspace{0.2cm}  \\
      \begin{block}{[cccc]c}
        0 & 0 & I & 0 & \hspace{0.3cm} 6 \\
        0 & 0 & 0 & I & \hspace{0.3cm} n_J \\
        \D_1 FD_b \circ \D L_H (I) 
        & \D_2 FD_b
        & \D_3 FD_b
        & \D_4 FD_b 
        & \hspace{0.3cm} 6 \\
        \D_1 FD_j \circ \D L_H (I) 
        & \D_2 FD_j
        & \D_3 FD_j
        & \D_4 FD_j
        & \hspace{0.3cm} n_J \\
      \end{block}
    \end{blockarray} .
    \label{eq:lefttrivtangent}
\end{align}
Note how, in \eqref{eq:lefttrivtangent}, 
the left-translated tangent maps of $FD_b$ and $FD_j$ with respect to $\vs$, $\vv$ and $\vr$ --entries $(3,2)$, $(3,3)$, $(3,4)$, $(4,2)$, $(4,3)$, and $(4,4)$-- are ordinary partial derivatives. 

Combining the definition of the state matrix \eqref{eq:matrixAgeneral}, the matrix form of the adjoint representation \eqref{eq:adjointA} and the left-trivialized tangent map \eqref{eq:lefttrivtangent}, we 
arrive at the expression for the state matrix $A(\eta,t)$ given in \eqref{eq:matrixA}.
The expression for the input matrix $B(\eta,t)$ given in \eqref{eq:matrixB}, derives from the identity
\begin{equation}
	\D_2 \lambda (g,u,t) = 			
	\begin{blockarray}{cc}
		n_J & \vspace{0.2cm} \\
		\begin{block}{[c]c}
			0 & \hspace{0.3cm} 6  \\
			0 & \hspace{0.3cm} n_J \\   
			\D_5 FD_b & \hspace{0.3cm} 6 \\
			\D_5 FD_j & \hspace{0.3cm} n_J \\
		\end{block} 
	\end{blockarray} .
\end{equation} $\hfill \Box$

The input matrix $B(\eta,t)$ given in \eqref{eq:matrixB} can be
computed explicitly as follows. The proof is ommitted as straightforward.

\begin{proposition}
    \label{proposition:inputmatrix}
    The input matrix $B(\eta,t)$ in \eqref{eq:matrixB} can be computed as
    \begin{equation} \label{eq:matrixBfinal}
        B(\eta,t) 
        = 
        \begin{blockarray}{cc}
          n_J & \vspace{0.2cm} \\
          \begin{block}{[c]c}
            0            & \hspace{0.3cm} 6  \\
            0            & \hspace{0.3cm} n_J \\  
            \vM^{-1} \vS & \hspace{0.3cm} 6 + n_J \\
          \end{block} 
        \end{blockarray} .
    \end{equation} 
    $\hfill \blacksquare$
\end{proposition}

\subsection{Forward dynamics linearization  via inverse dynamics}

In this section, we show that the geometric
linearization state and input matrices  \eqref{eq:matrixA} and \eqref{eq:matrixB} 
can be rewritten in terms of (the derivatives of) the inverse dynamics.
This is a not trivial task, in particular
as moving-base systems are underactuated and thus, strictly speaking, the inverse dynamics is not defined for arbitrary combinations of the system position, velocity, and acceleration. 
We can, however, provide moving-base systems with a non-physical\footnote{a technique that can be used
also in dynamic inversion for underactuated systems~\cite{saccon2012trajectory}} control input
, $\bar{\vtau}_b \in \R^6$, to obtain an (artificial) fully actuated system which we name the \textit{extended system}. 
This non-physical control input can be physically interpreted as a virtual wrench applied to the moving base and that controls all its degrees of freedom, turning the system into a fully actuated one. 
To this end, we define the \textit{extended input vector} $\bar{\vtau} \in \R^n$ as
\begin{equation} \label{eq:tauextended}
    \bar{\vtau} 
    := 
    \begin{bmatrix}
        \bar{\vtau}_b \\
        \vtau
    \end{bmatrix} ,
\end{equation}
and consequently the \textit{extended dynamics} as 
\begin{equation} \label{eq:dynamicsbiasextended}
    \vM \dot{\vnu} + \vh = \bar{\vtau} .
\end{equation}
Willing to partition \eqref{eq:dynamicsbiasextended} into 
base and joint dynamics, we can then rewrite it as
\begin{equation}
    \begin{bmatrix} 
        \vM_{11} & \vM_{12} \\
        \vM_{21} & \vM_{22}
    \end{bmatrix}
    \begin{bmatrix}
        \dot{\vv} \\
        \dot{\vr}
    \end{bmatrix} 
    +
    \begin{bmatrix}
        \vh_1 \\
        \vh_2
    \end{bmatrix}
    =
    \begin{bmatrix} 
        \bar{\vtau}_b \\
        \vtau
    \end{bmatrix} .
\end{equation}
More abstractly, \eqref{eq:dynamicsbiasextended}
can also be written similarly to \eqref{eq:dynsysgeneral} as
\begin{equation} \label{eq:dynsysgeneralextended}
    \dot{g}(t) 
    =
    \bar{f}(g,\bar{u},t),
\end{equation}
where as before $
    g(t) 
    =
    \left( \vH(t), \vs(t), \vv(t), \vr(t) \right)$
, while $
    \bar{f} 
    :
    G \times \R^n \times R 
    \rightarrow 
    TG$
, $
    (g,\bar{u},t) 
    \mapsto 
    \bar{f}(g,\bar{u},t)$ 
and $
    \bar{u} 
    =
    \bar{\vtau} \in \R^n$. 
We define the extended forward dynamics $
    \overline{FD} 
    :
    G \times \R^6 \times \R^{n_J} 
    \rightarrow
    \R^n$ 
of the extended system \eqref{eq:dynamicsbiasextended} as
\begin{align} 
    \overline{FD}(\vH,\vs,\vv,\vr, \bar{\vtau}_b, \vtau) 
    & :=
    \dot{\vnu}(\vH,\vs,\vv,\vr, \bar{\vtau}_b, \vtau) 
    \nonumber
    \\
    & \phantom{:}=
    \vM^{-1}(\vs) \big[ - \vh(\vH,\vs,\vv,\vr) + \bar{\vtau}(\bar{\vtau_b},\vtau) \big] ,
    \label{eq:extendedForwardDynamics}
\end{align}
which can be partition, similarly to \eqref{eq:forwarddynamicsbase} and \eqref{eq:forwarddynamicsjoints}, as 
$
    \overline{FD} 
    =
    [\overline{FD}_b ; \overline{FD}_j]
$ 
where
\begin{align}
    \overline{FD}_b (\vH,\vs,\vv,\vr, \bar{\vtau}_b, \vtau) 
    & :=
    \dot{\vv}(\vH,\vs,\vv,\vr, \bar{\vtau}_b, \vtau) 
    \\
    \overline{FD}_j (\vH,\vs,\vv,\vr, \bar{\vtau}_b, \vtau) 
    & :=
    \dot{\vr}(\vH,\vs,\vv,\vr, \bar{\vtau}_b, \vtau)
    , 
\end{align}
with 
$
    \overline{FD}_b 
    :
    G \times \R^6 \times \R^{n_J} 
    \rightarrow 
    \R^6
$ 
and 
$
    \overline{FD}_j 
    : 
    G \times \R^6 \times \R^{n_J} 
    \rightarrow 
    \R^{n_J}.
$ 
Because the extended system \eqref{eq:dynamicsbiasextended} is fully actuated, its inverse dynamics 
$
    \overline{ID} 
    :
    G \times \R^6 \times \R^{n_J} 
    \rightarrow
    \R^n
$
is always well defined and reads
\begin{equation} \label{eq:inversedynamics}
    \overline{ID}(\vH,\vs,\vv,\vr,\dot{\vv},\dot{\vr}) 
    := 
    \bar{\vtau}(\vH,\vs,\vv,\vr,\dot{\vv},\dot{\vr}) 
    =
    \vM(\vs) \: \dot{\vnu}(\dot{\vv},\dot{\vr}) + \vh(\vH,\vs,\vv,\vr) ,
\end{equation}
which can be partitioned as $
    \overline{ID} 
    =
    [\overline{ID}_b ; \overline{ID}_j]$ 
with
\begin{align} 
    \overline{ID}_b (\vH,\vs,\vv,\vr,\dot{\vv},\dot{\vr}) 
    &\phantom{:}=
    \vM_{11}(\vs) \dot{\vv} + \vM_{12}(\vs) \dot{\vr} + \vh_1 (\vH,\vs,\vv,\vr) 
    \label{eq:inversedynamicsbase} 
    \\
    \overline{ID}_j (\vH,\vs,\vv,\vr,\dot{\vv},\dot{\vr}) 
    &\phantom{:}=
    \vM_{21}(\vs) \dot{\vv} + \vM_{22}(\vs) \dot{\vr} + \vh_2 (\vH,\vs,\vv,\vr) 
    \label{eq:inversedynamicsjoints} ,
\end{align}
where 
$
    \overline{ID}_b 
    :
    G \times \R^6 \times \R^{n_J} 
    \rightarrow
    \R^6
$ 
and 
$
    \overline{ID}_j 
    : 
    G \times \R^6 \times \R^{n_J} 
    \rightarrow
    \R^{n_J} .
$ 

We can now show that the state matrix of the forward dynamics linearization \eqref{eq:matrixA} can be expressed in terms of the extended inverse dynamics derivatives and the inverse of the mass matrix.

\begin{proposition}
    \label{proposition:inverserelation}
    The geometric linearization state matrix $A(\eta ,t)$ 
    in \eqref{eq:matrixA} can be rewritten in terms of the derivatives of the extended inverse dynamics \eqref{eq:inversedynamics} and the inverse of the mass matrix $\vM$ as
    \begin{align} 
        &A(\eta,t) 
        =
        \nonumber
        \\
        &\begin{blockarray}{ccccc}
          6 & n_J & 6 & n_J & \vspace{0.2cm} \\
          \begin{block}{[cccc]c}
            - \vv \times & 0 & I & 0 & \hspace{0.3cm} 6   \\
                       0 & 0 & 0 & I & \hspace{0.3cm} n_J \\
            - \vM^{-1} \D_1 \overline{ID} \circ \D L_H (I) 
            & 
            - \vM^{-1} \D_2 \overline{ID}
            &
            - \vM^{-1} \D_3 \overline{ID}
            &
            - \vM^{-1} \D_4 \overline{ID}
            & \hspace{0.3cm} 6 + n_J \\
          \end{block}
        \end{blockarray} ,
        \label{eq:matrixAinvdyn}
    \end{align}
    where $\overline{ID}$ is the extended inverse dynamics as in \eqref{eq:inversedynamics} and its derivatives are evaluated at $\overline{ID}( \vH, \vs, \vv, \vr, \dot{\vv}, \dot{\vr} )$. 
    $\hfill \blacksquare$
\end{proposition}

\noindent
\textit{Proof:} The extended forward and extended inverse dynamics are related to each other through
\begin{equation} \label{eq:IDFDidentitygeneral}
    \overline{ID} \circ \overline{FD} 
    =
    id.
\end{equation}
Evaluated at an arbitrary extended joint torques vector $\bar{\vtau} \in \R^n$, \eqref{eq:IDFDidentitygeneral} reads
\begin{equation} 
    \overline{ID} \big( \vH, \vs, \vv, \vr,
    \overline{FD}_b(\vH,\vs,\vv,\vr, \bar{\vtau}_b, \vtau),
    \overline{FD}_j(\vH,\vs,\vv,\vr, \bar{\vtau}_b, \vtau)
    \big) 
    = 
    \bar{\vtau} ,
\end{equation}
which in extended, matrix form is given by
\begin{equation}
    \begin{bmatrix} \label{eq:IDFDidentity}
        \overline{ID}_b \big( \vH, \vs, \vv, \vr,
        \overline{FD}_b (\vH,\vs,\vv,\vr,\bar{\vtau}_b,\vtau),
        \overline{FD}_j (\vH,\vs,\vv,\vr,\bar{\vtau}_b,\vtau)
        \big) \\
        \overline{ID}_j \big( \vH, \vs, \vv, \vr,
        \overline{FD}_b (\vH,\vs,\vv,\vr,\bar{\vtau}_b,\vtau),
        \overline{FD}_j (\vH,\vs,\vv,\vr,\bar{\vtau}_b,\vtau)
        \big)
    \end{bmatrix}
    = 
    \begin{bmatrix}
        \bar{\vtau}_b \\
        \vtau
    \end{bmatrix} .
\end{equation}
For the sake of readability, we define 
\begin{align}
	\dot{\vv} 
	&:=
	\overline{FD}_b(\vH,\vs,\vv,\vr, \bar{\vtau}_b,\vtau) \quad\text{and}
	\\
	\dot{\vr} 
	&:=
	\overline{FD}_j(\vH,\vs,\vv,\vr, \bar{\vtau}_b,\vtau) .
\end{align}
Differentiating \eqref{eq:IDFDidentity} with respect to $\vH, \vs, \vv$ and $\vr$,  applying the chain rule, we get
\begin{equation} \label{eq:IDFDrelation}
    \begin{bmatrix}
        \D_i \overline{ID}_b \\
        \D_i \overline{ID}_j
    \end{bmatrix}
    +
    \begin{bmatrix}
    	\D_5 \overline{ID}_b & \D_6 \overline{ID}_b \\
    	\D_5 \overline{ID}_j & \D_6 \overline{ID}_j
    \end{bmatrix} 
    \begin{bmatrix}
    	\D_i \overline{FD}_b \\
    	\D_i \overline{FD}_j
    \end{bmatrix}
    =
    0 ,
\end{equation}
with $i \in \{1, 2, 3, 4\}$, as $\D_i \overline{ID}$ represents the left-trivialized partial derivative of the extended inverse dynamics with respect to the four system state variables. 
In \eqref{eq:IDFDrelation} the arguments of the inverse and forward dynamics functions are 
$\overline{ID}_b(\vH,\vs,\vv,\vr,\dot{\vv},\dot{\vr})$, 
$\overline{ID}_j(\vH,\vs,\vv,\vr,\dot{\vv},\dot{\vr})$, 
$\overline{FD}_b(\vH,\vs,\vv,\vr, \bar{\vtau}_b,\vtau)$ and 
$\overline{FD}_j(\vH,\vs,\vv,\vr, \bar{\vtau}_b,\vtau)$, 
which have been left out for the sake of readability. 
From \eqref{eq:inversedynamicsbase} and \eqref{eq:inversedynamicsjoints}, it follows that
\begin{equation} \label{eq:IDderivative}
    \begin{bmatrix}
    	\D_5 \overline{ID}_b & \D_6 \overline{ID}_b \\
    	\D_5 \overline{ID}_j & \D_6 \overline{ID}_j
    \end{bmatrix} 
    =
    \begin{bmatrix}
        \vM_{11} & \vM_{12} \\
        \vM_{21} & \vM_{22}
    \end{bmatrix}
    =
    \vM
\end{equation}
Substituting \eqref{eq:IDderivative} into \eqref{eq:IDFDrelation} and pre-multiplying with $\vM^{-1}$ gives
\begin{equation} 
	\begin{bmatrix} 
	    \D_i \overline{FD}_b(\vH,\vs,\vv,\vr, \bar{\vtau}_b,\vtau) \\
	    \D_i \overline{FD}_j(\vH,\vs,\vv,\vr, \bar{\vtau}_b,\vtau)
	\end{bmatrix} 
	=
	- \vM^{-1} 
	\begin{bmatrix}
	    \D_i \overline{ID}_b(\vH,\vs,\vv,\vr,	    \dot{\vv},\dot{\vr} ) \\
	    \D_i \overline{ID}_j(\vH,\vs,\vv,\vr,	    \dot{\vv},\dot{\vr} )
	\end{bmatrix} . \label{eq:IDFDrelationextended}
\end{equation}
The equality \eqref{eq:IDFDrelationextended} can be used within the original underactuated dynamics \eqref{eq:dynamicsbias} where, due to the definitions in \eqref{eq:forwarddynamics} and \eqref{eq:extendedForwardDynamics}, $\bar{\vtau}_b = 0_{6 \times 1}$ ensuring
\begin{align}
    FD_b (\vH,\vs,\vv,\vr,\vtau) 
    &=
    \overline{FD}_b (\vH,\vs,\vv,\vr,0_{6 \times 1},\vtau) \label{eq:FDvbar} \quad\text{and}
    \\
    FD_j (\vH,\vs,\vv,\vr,\vtau) 
    &=
    \overline{FD}_j   (\vH,\vs,\vv,\vr,0_{6 \times 1},\vtau) . \label{eq:FDrbar}
\end{align}
Substituting \eqref{eq:FDvbar} and \eqref{eq:FDrbar} on the left-hand side in \eqref{eq:IDFDrelationextended} results in the relation between the forward dynamics of the original moving-base system and extended inverse dynamics 
\begin{align} 
    \begin{bmatrix} 
	    \D_i FD_b(\vH,\vs,\vv,\vr,\vtau) \\
	    \D_i FD_j(\vH,\vs,\vv,\vr,\vtau)
	\end{bmatrix} 
	&=
	- \vM^{-1}
	\begin{bmatrix}
	    \D_i \overline{ID}_b 
	    \\
	    \D_i \overline{ID}_j 
	\end{bmatrix} 
	=
	- \vM^{-1}
	    \D_i \overline{ID} 
	\label{eq:IDFDrelationfinal}
\end{align}
where the inverse dynamics derivatives are 
evaluated at $(\vH,\vs,\vv,\vr,\dot{\vv},\dot{\vr})$.
Combining the original expression for $A(\eta,t)$ given by \eqref{eq:matrixA} with the relationship \eqref{eq:IDFDrelationfinal} relating the derivatives of the forward dynamics with those of the extended inverse dynamics, we obtain \eqref{eq:matrixAinvdyn}. $\hfill \Box$ 
\\

We conclude this section with an important observation and a definition that will be used in the following section. Given arbitrary values for $\vH, \vs, \vv, \vr, \dot{\vv}$ and $\dot{\vr}$, the extended inverse dynamics function $\overline{ID}$ in \eqref{eq:inversedynamics} will generally return a non-zero base torque $\bar{\vtau}_b$. Those combination of values for which the moving base torque is zero will be called \textit{consistent}:
\begin{definition}
    \label{def:consistent} 
    The inputs $\vH, \vs, \vv, \vr, \dot{\vv}$ and $\dot{\vr}$ of the extended inverse dynamics function given by \eqref{eq:inversedynamics} are called \textit{consistent} only if the extended inverse dynamics function returns a vector with the first six elements equal to zero. 
    $\hfill \blacksquare$
\end{definition}
When the inputs represent a \textit{physical} system (i.e. the sensors obtain data from a physical system or the inputs are numerically obtained through simulations), the inputs $\vH, \vs, \vv, \vr, \dot{\vv}$ and $\dot{\vr}$ of the extended inverse dynamics \eqref{eq:inversedynamics} are consistent (neglecting sensor noise or rounding errors). For this reason, commonly used moving-base algorithms in robot control such as the Recursive Newton Euler Algorithm for moving-base systems (RNEAmb) \cite[Table 9.6]{featherstone2008rigid} and Articulated Body Algorithm for moving-base systems (ABAmb) \cite[Table 9.4]{featherstone2008rigid} assume the inputs to be consistent in deriving and simplify the algorithm steps. Such a consistency assumption is however not valid for our goal, specifically because we need to compute derivatives by allowing for arbitrary nonphysically perturbation directions. For this reason, new extended versions of these algorithms are necessary and are developed in the following section.

\section{Left-Trivialized Linearization: recursive algorithm}
\label{chp:algorithmicaspects}
In the previous section and, more in particular, in Proposition~\ref{proposition:inverserelation}, we have demonstrated how the computation of the 
left-trivialize linearization (cf. Proposition~\ref{proposition:linearization})
can be obtain once the
derivative of the 
(extended) inverse dynamics and the inverse of the mass matrix are known.
This section provides new recursive algorithms for this purpose. More specifically, the section is structured in as follows. In Subsection~\ref{sec:listofsymbolsalgos} we provide a list of symbols commonly used in all algorithms. In Subsection~\ref{sec:EIDA}, we present a recursive algorithm
for the extended Inverse dynamics. An algorithm for computing its derivatives is then provided in Subsection~\ref{sec:diffEIDA}. The section concludes with the 
inverse mass matrix algorithm for moving-base multibody systems, presented in 
Subsection~\ref{sec:IMMAmb}. A sufficient familiarity with the original recursive  algorithms presented in  \cite{featherstone2008rigid} is a prerequisite for being able to fully appreciate the content of this section.

\subsection{List of symbols for the recursive algorithms}
\label{sec:listofsymbolsalgos}

The kinematic and dynamic quantities presented in this subsection, and in particular in Table~\ref{tab:rigidbodydynamics},  are used throughout all algorithms presented in this section. 
The notation employed is based on the coordinate frames and multibody notation 
detailed in \cite{traversaro2019multibodyv2} and concisely reported also in Section~\ref{chp:preliminaries}.
It is suggested to skip this section at first read (or to read it diagonally) and to return to it when deemed necessary for the understanding of the algorithm details.

\noindent The following shorthand notation is also employed to improve readability:
\begingroup
\renewcommand{\arraystretch}{1.2}
\begin{table}[H]
\centering
\begin{tabular}{llll}
$\vv_i       = \ls^i\vv_{A,i} \quad$
& $\vv_{\cJ i} = \ls^i\vv_{\lambda(i),i} \quad$
&  $\quad \vf_{\cB i}  = \ls_i\vf_{\cB i}$ 
&  $\quad \bbM_{\cB i} = \ls_i^\pc\!\!\bbM^{\cB i}_i$
\\
$\va_i       = \ls^i\va_{A,i} \quad$
& $\vc_{\cJ i} = \ls^i\vc_{\lambda(i),i} \quad$ 
& $\quad \vf_{\cE i}  = \ls_i\vf_{\cE i}$
& $\quad \vb_{\cB i}  = \ls_i\vb_{\cB i}$
\\
$\vm_{\cB i} = \ls_i\vm_{\cB i} \quad$ 
& $\vd_{\cJ i} = \ls^i\vd_{\lambda(i),i} \quad$ 
& $\quad \vf_{\cI i}  = \ls_i\vf_{\cI i}$
& $\quad \vU_{\cB i}  = \ls_i\vU_{\cB i}$
\\
& $\vGamma_{\cJ i} = \ls^i\vGamma_{\lambda(i),i} \quad$
& $\quad \vf_{\cJ i}  = \ls_i\vf_{\cJ i}$ 
&  
\end{tabular}
\end{table}
\endgroup

\subsection{Extended Inverse Dynamics Algorithm}
\label{sec:EIDA}

In what follows we detail the 
Extended Inverse Dynamics Algorithm for moving-base systems (EIDAmb),
used to compute the extended inverse dynamics $\overline{ID}$ given by \eqref{eq:inversedynamics}. The original RNEAmb presented in \cite[Section 5.3]{featherstone2008rigid} only partly fulfills this purpose, as it computes the joint torques $\vtau$ under the assumption that its inputs are consistent (see Definition \ref{def:consistent}). 
By adding lines of the Generalized Bias Wrench Algorithm for moving-base systems (GBWAmb) \cite[Table 9.5]{featherstone2008rigid} and the Composite Rigid Body Algorithm for moving-base systems (CRBAmb) \cite[Table 9.5]{featherstone2008rigid}, we can modify the RNEAmb to also compute the non-physical control input $\bar{\vtau}_b$, so that the extended input vector $\bar{\vtau}$ \eqref{eq:tauextended} can be constructed.

From \eqref{eq:inversedynamics}, we see that $\bar{\vtau}_b = \vM_{11} \dot{\vv} + \vM_{12} \dot{\vr} + \vh_1$. 
The variables $\vM_{11}, \vM_{12}, \dot{\vv}$ and $\vh_1$ can be translated to rigid body dynamics variables used in recursive algorithms as
\begin{align}
    \vM_{11} &= \bbM^c_{\cB 0} , \\
    \dot{\vv} &= \va_0 , \label{eq:accelerations} \\
    \vM_{12} &= \ls_0\vF \quad \text{and}
    \\
    \vh_1 &= \vb^{vp}_{\cB 0} ,
\end{align}
where for \eqref{eq:accelerations}, we made use of the relation between intrinsic and apparent acceleration as seen in \cite[Section 5.4]{traversaro2019multibodyv2}.
In terms of the rigid body dynamics variables, the non-physical control input $\bar{\vtau}_b$ can be thus written as
\begin{equation} \label{eq:inversedynamicsbaseRNEAmb}
    \bar{\vtau}_b (\vH,\vs,\vv,\vr,\dot{\vv},\dot{\vr}) 
    =
    \bbM^c_{\cB 0}(\vs) \va_0(\dot{\vv}) 
    +
    \ls_0\vF(\vs) \dot{\vr} 
    +
    \vb^{vp}_{\cB 0}(\vH,\vs,\vv,\vr,\dot{\vv},\dot{\vr}) ,
\end{equation}
where $\vb^{vp}_{\cB 0} \in \R^6$ is the bias wrench of body 0 with zero joint acceleration and 
$\ls_0\vF = [\ls_0\vF_{\cB 1} \enspace \ls_0\vF_{\cB 2} \; ... \; \ls_0\vF_{\cB n_J}] \in \R^{6 \times n_J}$ the required wrenches to support unit accelerations.
Note that $\ls_0\vF_{\cB i}$ is defined as
\begin{align} \label{eq:defineF}
    \ls_i\vF_{\cB i} 
    :=&
    \bbM^c_{\cB i} \vGamma_{\cJ i} \in \R^6 , \\
    \ls_0\vF_{\cB i} 
    =&
    \ls_0\vX^i \ls_i\vF_{\cB i} .
\end{align}
In the EIDAmb, the acceleration term $\va_i^r$ is defined as 
\begin{equation}
    \va_i^r 
    :=
    \ls^i\va_{0,i} + \ls^i\va_{grav} ,
\end{equation}
thus $\va_i^r$ is the acceleration of body $i$ relative to the moving-base frame $0$ \emph{plus} the gravitational acceleration. 
This acceleration is then used to compute the bias wrench, which is later used to compute the acceleration $\va_{A,i}$ of body $i$ with respect to the inertial frame $A$.  

\begingroup
\renewcommand{\arraystretch}{1.2}
\begin{table}[H]
\caption{Kinematic, dynamics, and set theoretic quantities used in the recursive algorithms presented in this section.}
\label{tab:rigidbodydynamics}
\centering
\begin{tabular}{|p{1.4cm}|p{1cm}|p{9.2cm}|}
\hline
\textbf{EG} & \textbf{Dim.} & \textbf{Explanation} \\ \hline
$\ls_C^\pc \bbM^{\cB i}_C$ 
   & $\R^{6 \times 6}$
   & Inertia matrix of body $i$
   \\ \hline
$\ls_C^\pc \bbM^{\cB i,A}_C$ 
   & $\R^{6 \times 6}$
   & Articulated-body inertia matrix of body $i$
   \\ \hline
$\ls_C^\pc \bbM^{\cB i,a}_C$ 
   & $\R^{6 \times 6}$
   & Apparent articulated-body inertia matrix of body $i$
   \\ \hline
$\ls_C^\pc \bbM^{\cB i,c}_C$ 
   & $\R^{6 \times 6}$
   & Composite rigid body inertia matrix of body $i$
   \\ \hline
$\ls^C\vv_{A,i}$  
   & $\R^6$
   & Twist or spatial velocity of frame $i$ w.r.t frame $A$
   \\ \hline
$\ls^C\dot{\vv}_{A,i}$ 
   & $\R^6$
   & Apparent acceleration of frame $i$ w.r.t frame $A$
   \\ \hline
$\ls^C\va_{A,i}$ 
   & $\R^6$
   & Intrinsic acceleration of frame $i$ w.r.t frame $A$
   \\ \hline
$\ls^C\va_{grav}$ 
   & $\R^6$
   & Intrinsic gravitational acceleration
   \\ \hline
$\ls^C\va_i^r$ 
   & $\R^6$
   & Intrinsic acceleration relative to the moving-base acceleration, plus the gravitational acceleration of body $i$
   \\ \hline
$\ls^C\va_i^{vp}$ 
   & $\R^6$
   & Intrinsic acceleration that only accounts for the velocity product terms of body $i$
   \\ \hline
$\ls_C\vm_{\cB i}$
   & $\R^6$
   & Spatial momentum of body $i$
   \\ \hline
$\ls^C\vGamma_{\lambda(i),i}$ 
   & $\R^6$
   & Joint velocity subspace matrix of joint $i$
   \\ \hline
$\ls^i\vv_{\lambda(i),i}$ 
   & $\R^6$
   & Velocity of joint $i$
   \\\hline
$\ls_C\vb_{\cB i}$ 
   & $\R^6$   
   & Bias wrench acting on body $i$
   \\ \hline   
$\ls_C\vb^c_{\cB i}$ 
   & $\R^6$  
   & Composite rigid body bias wrench acting on body $i$
   \\ \hline
$\ls_C\vb^{vp}_{\cB 0}$
   & $\R^6$ 
   & Bias wrench of the moving-base with zero joint acceleration
   \\ \hline
$\ls^C\vX_D$ 
   & $\R^{6 \times 6}$
   & Velocity transformation from frame $D$ to frame $C$
   \\ \hline
$\ls_C\vX^D$ 
   & $\R^{6 \times 6}$
   & Wrench transformation from frame $D$ to frame $C$
   \\ \hline
$\ls_C\vU_{\cB i}$ 
   & $\R^6$
   & Subexpression used in ABA
   \\ \hline
$\vD_{\cB i}$ 
   & $\R$
   & Subexpression used in ABA
   \\ \hline
$\vu_{\cB i}$ 
   & $\R$
   & Subexpression used in ABA
   \\ \hline
$\ls_C\vF_{\cB i}$ 
   & $\R^6$
   & Required wrench to support unit acceleration of joint $i$
   \\ \hline
$\mathcal{F}_i$
   & $\R^{6 \times n_J}$
   & Wrench set collecting the contributions of the supporting tree rooted at $i$
   \\ \hline
$\mathcal{P}_i$
   & $\R^{6 \times n_J}$
   & Motion set which contains the contributions of all parents of joint $i$
   \\ \hline
\end{tabular}
\end{table}
\endgroup

Table~\ref{tab:EIDAmb} presents the full algorithm\footnote{We use $[i,j]$ and alike to indicate specific submatrices. For example, $\vH[1\!\!:\!\!3,4]$ means we consider the column vector obtained
taking the first three elements of the fourth column of the $4\times4$ matrix $\vH$.} (EIDAmb).
The function jcalc(jtype(i),$\vs_i$) appearing in the the algorithm acts as a look-up table, from where joint-type-specific functions are retrieved. 
These joint-type-specific functions compute the velocity transformation matrix $\ls^i\vX_{\lambda(i|i}$ and the joint velocity subspace matrix $\vGamma_{\cJ i}$. 
The explicit joint-type-specific functions can be found in \cite[Appendix B.2]{bos2019} or can be straightforwardly derived by using, e.g., \cite[Section 4.4]{featherstone2008rigid}.

\begingroup
\renewcommand{\arraystretch}{1.2}
\begin{table}[H]
\caption{EIDAmb.}
\label{tab:EIDAmb}
\centering
\begin{tabular}{|p{0.8cm}|p{5.8cm}|}
\hline
\multicolumn{2}{|l|}{\textbf{Inputs:} 
   model, $\vs, \vr, \dot{\vr}, \ls^A\vH_0, \ls^A\vv_{A,0}, \ls^0\va_{A,0}, \ls^A\va_{grav}$}
   \\ \hline
\textbf{Line} & \textbf{EIDAmb} \\ \hline
1  & $\ls^0\vH_A = \ls^A\vH_0^{-1}$
   \\
2  & $\ls^0\vR_A = \ls^0\vH_A[1\!\!:\!\!3,1\!\!:\!\!3]$
   \\
3  & $\ls^0\vo_A = \ls^0\vH_A[1\!\!:\!\!3,4]$
   \\
4  & $\ls^0\vX_A = 
     \begin{bmatrix}
       \ls^0\vR_A & \ls^0\vo_A^\wedge \ls^0\vR_A \\
       0_{3\times3} & \ls^0\vR_A
    \end{bmatrix}$
   \\   
5  & $\vv_0 = \ls^0\vX_A \ls^A\vv_{A,0}$
   \\
6  & $\va_0^r = \ls^0\vX_A \ls^A\va_{grav}$
   \\
7* & $\va_0^{vp} = \va_0^r$
   \\
8  & $\bbM_{\cB 0}^c = \bbM_{\cB 0}$
   \\
9 & $\vm_{\cB 0} = \bbM_{\cB 0} \vv_0$
   \\   
10 & $\vb_{\cB 0}^c = \bbM_{\cB 0} \va_0^r + \vv_0 \bar{\times}^* 
      \vm_{\cB 0}$ 
   \\ 
11*& $\vb_{\cB 0}^{vp} = \vb_{\cB 0}^c$ 
   \\
12 & $\mathbf{for}\ i=1\ \mathbf{to}\ n_B\ \mathbf{do}$  
   \\
13 & $\quad [ \ls^i\vX_{\lambda(i)|i}, \vGamma_{\cJ i} ] = $
     jcalc(jtype($i$)$, \vs_i$) 
   \\
14 & $\quad \vv_{\cJ i} = \vGamma_{\cJ i} \vr_i$ 
   \\
15 & $\quad \ls^i\vX_{\lambda(i)} = \ls^i\vX_{\lambda(i)|i} 
     \ls^{\lambda(i)|i}\vX_{\lambda(i)}$ 
   \\
16 & $\quad \vv_i = \ls^i\vX_{\lambda(i)} \vv_{\lambda(i)} + \vv_{\cJ i}$ 
   \\
17 & $\quad \va_i^r = \ls^i\vX_{\lambda(i)} \va_{\lambda(i)}^r + 
     \vGamma_{\cJ i} \dot{\vr}_i + \vv_i \times \vv_{\cJ i}$ 
   \\
18*& $\quad \va_i^{vp} = \ls^i\vX_{\lambda(i)} \va_{\lambda(i)}^{vp} + 
     \vv_i \times \vv_{\cJ i}$ 
   \\
19 & $\quad \bbM_{\cB i}^c = \bbM_{\cB i}$
   \\
20 & $\quad \vm_{\cB i} = \bbM_{\cB i} \vv_i$
   \\   
21 & $\quad \vb_{\cB i}^c = \bbM_{\cB i} \va_i^r + \vv_i \bar{\times}^* 
      \vm_{\cB i}$ 
   \\ 
22*& $\quad \vb_{\cB i}^{vp} = \bbM_{\cB i} \va_i^{vp} + 
      \vv_i \bar{\times}^* \vm_{\cB i}$ 
   \\ 
23 & $\mathbf{end}$
   \\
24 & $\mathbf{for}\ i=n_B\ \mathbf{to}\ 1\ \mathbf{do}$ 
   \\
25 & $\quad \bbM_{\cB \lambda(i)}^c = \bbM_{\cB \lambda(i)}^c +
     \ls_{\lambda(i)}\vX^i ~ \bbM_{\cB i}^c \ls^i\vX_{\lambda(i)}$
   \\
26 & $\quad \vb_{\cB \lambda(i)}^c = \vb_{\cB \lambda(i)}^c + 
     \ls_{\lambda(i)}\vX^i ~ \vb_{\cB i}^c$
   \\
27*& $\quad \vb_{\cB \lambda(i)}^{vp} = \vb_{\cB \lambda(i)}^{vp} + 
     \ls_{\lambda(i)}\vX^i ~ \vb_{\cB i}^{vp}$
   \\
28 & $\mathbf{end}$ 
   \\   
\hline
\multicolumn{2}{|c|}{\textit{continued on the next page.}}
   \\
\hline
\end{tabular}
\end{table}
\endgroup

\begingroup
\renewcommand{\arraystretch}{1.2}
\begin{table}[H]
\centering
\begin{tabular}{|p{0.8cm}|p{5.8cm}|}
\hline
\multicolumn{2}{|c|}{\textit{continued from the previous page.}}
   \\
\hline   
29 & $\mathbf{for}\ i=1\ \mathbf{to}\ n_B\ \mathbf{do}$  
   \\
30 & $\quad \ls^i\va_{A,0} = \ls^i\vX_{\lambda(i)}
     \ls^{\lambda(i)}\va_{A,0}$
   \\
31 & $\quad \vtau_i = \vGamma_{\cJ i}^T ( \bbM_{\cB i}^c \ls^i\va_{A,0} + 
     \vb_{\cB i}^c )$
   \\ 
32** & $\quad \ls_i\vF_{\cB i} = \bbM_{\cB i}^c \vGamma_{\cJ i}$
   \\
33** & $\quad j = i$
   \\   
34** & $\quad \mathbf{while}\ \lambda(j) > 0$
   \\ 
35** & $\qquad \ls_{\lambda(j)}\vF_{\cB i} = \ls_{\lambda(j)}\vX^j \ls_j\vF_{\cB i}$
   \\
36** & $\qquad j = \lambda(j)$
   \\ 
37** & $\quad \mathbf{end}$
   \\
38** & $\quad \ls_ 0\vF_{\cB i} = \ls_0\vX^j \ls_j\vF_{\cB i}$
   \\
39 & $\mathbf{end}$ 
   \\ 
40***& $\bar{\vtau}_b = \bbM_{\cB 0}^c \ls^0\va_{A,0} + 
      \ls_0\vF \dot{\vr} + \vb_{\cB 0}^{vp}$
   \\ \hline
\multicolumn{2}{|l|}{\textbf{Output:} 
   $\overline{ID} = \bar{\vtau} = [\vtau ; \bar{\vtau}_b]$}
   \\ \hline
\end{tabular}
\end{table}
\endgroup

\noindent
\textit{Remarks:} 
\begin{itemize}
    \item Lines marked with * origin from the GBWAmb, lines marked with ** origin from the CRBAmb, and unmarked lines origin from the RNEAmb. 
    \item The line marked with *** computes the extended inverse dynamics of the moving base as given in \eqref{eq:inversedynamicsbaseRNEAmb}.
    \item Arbitrary inputs are generally not consistent (see Definition \ref{def:consistent}). 
    If one wants the inputs of the extended inverse dynamics function \eqref{eq:inversedynamics} to be consistent, one can add the computation $\va_0 = - ( \bbM_{\cB 0}^c )^{-1} \vb_{\cB 0}^c$ between line 28 and 29 instead of using it as an input to the algorithm. 
    This guarantees the inputs of the extended inverse dynamics function to be consistent, at the cost of losing the freedom to choose $\va_0$. 
    \item External wrenches are neglected in the algorithms presented above (see Assumption \ref{assumption:externalforces}).
    \item The velocity of the moving base as input of the EIDAmb, $\ls^A\vv_{A,0}$ is chosen to be expressed in frame $A$, as it is often measured in that frame. 
    The acceleration of the moving base as input of the EIDAmb, $\ls^0\va_{A,0}$, is chosen to be expressed in frame $0$, as it is often computed through the RNEAmb in frame $0$. 
    If one desires to express either of both quantities in a different frame, the algorithm can be straightforwardly modified. 
\end{itemize}

\subsection{Computing the derivatives of the extended inverse dynamics}
\label{sec:diffEIDA}

This section proposes \emph{four} new recursive algorithms that  
compute the left-trivialized derivatives of the extended inverse dynamics \eqref{eq:inversedynamics} with respect to $\vH, \vs, \vv$ and $\vr$, respectively. 
To improve the readability and reduce the heaviness of the expressions, we employ a shorthand notation. Namely,
the left-trivialized derivative of an arbitrary function $x(\vH,\vs,\vv,\vr)$ with respect to $\vH$ is written in shorthand notation as 
$
    \tp x/\partial \vH :=
    \D_1 x(\vH,\vs,\vv,\vr) \circ \D L_H (I)
    , 
$
that is, in words, $\tp x/\partial g$ is defined as the \textit{left-trivialized derivative of the function $x$ with respect to Lie group element $g$}. 
Likewise, the partial derivatives (on vector spaces) of an arbitrary function $x(\vH,\vs,\vv,\vr)$ with respect to $\vs, \vv$ and $\vr$ are written in shortened notation as 
$
    {\partial x}/{\partial \vs} :=
    \D_2 x(\vH,\vs,\vv,\vr) , 
    {\partial x}/{\partial \vv} :=
    \D_3 x(\vH,\vs,\vv,\vr) 
$ and $
    {\partial x}/{\partial \vr} :=
    \D_4 x(\vH,\vs,\vv,\vr) ,
$
so that ${\partial x}/{\partial y}$ is defined as the standard \textit{partial derivative of $x$ with respect to the scalar (or vector) quantity $y$}.

In the left-trivialized derivatives algorithms, the computation of derivatives of the matrices $\ls^i\vX_{\lambda(i)|i}, \vF_{\cB i}$ and $\bbM^c_{\cB \lambda(i)}$ is required, which would result in 3D tensors (requiring 3 indeces). Because these matrices are only dependent on vector $\vs$, we choose to compute their derivatives with respect to the scalars $\vs_k$ for $k \in \R^{n_J}$, resulting in 2D tensors for which we can employ  standard matrix multiplication, without the need to resort to the more advanced tensor calculus notation.
As $\ls^i\vX_{\lambda(i)|i}$ is only dependent on the $i$-th entry of $\vs$, we only compute ${\partial \ls^i\vX_{\lambda(i)|i}}/{\partial \vs_i}$, which saves computational time.

Using the dependencies in the expression for the non-physical control input $\bar{\vtau}_b$ of \eqref{eq:inversedynamicsbaseRNEAmb}, we can omit unnecessary computations by computing the derivatives as 
\begin{align}
    \frac{\tp \bar{\vtau}_b}{\partial \vH} =&
        \frac{\tp \vb^{vp}_{\cB 0}}{\partial \vH} =
        \frac{\tp \vb^c_{\cB 0}}{\partial \vH}
        \label{eq:ltlinverseH} , \\
    \frac{\partial \bar{\vtau}_b}{\partial \vs} =&
        \frac{\partial \bbM^c_{\cB 0}}{\partial \vs} \va_0 
        +  \frac{\partial \ls_0\vF}{\partial \vs} \dot{\vr} 
        + \frac{\partial \vb^{vp}_{\cB 0}}{\partial \vs} 
        \label{eq:ltlinverses} , \\
    \frac{\partial \bar{\vtau}_b}{\partial \vv} =&
        \frac{\partial \vb^{vp}_{\cB 0}}{\partial \vv} =
        \frac{\partial \vb^c_{\cB 0}}{\partial \vv} 
        \label{eq:ltlinversev}  \\
    \frac{\partial \bar{\vtau}_b}{\partial \vr} =&
        \frac{\partial \vb^{vp}_{\cB 0}}{\partial \vr} =
        \frac{\partial \vb^c_{\cB 0}}{\partial \vr}
        \label{eq:ltlinverser} .
\end{align}
The last steps in \eqref{eq:ltlinverseH}, \eqref{eq:ltlinversev} and \eqref{eq:ltlinverser} are derived from the equality 
\begin{equation}
    \dfrac{\tp \va^r_i}{\partial \vH} = \dfrac{\tp \va^{vp}_i}{\partial \vH} ,
\end{equation}
which also holds for the derivatives with respect to $\vv$ and $\vr$, even though $\va^r_i$ is not generally equal to $\va^{vp}_i$.
\\

\noindent\textbf{First algorithm ($\vH$).} We start discussing the preliminaries required to understand the first algorithm of the four, the one that computes the left-trivialized derivative of the extended inverse dynamics with respect to the transformation matrix $\vH ( = \ls ^A\vH_0 )$. This is the most complicated of all four. The complexity lays, in particular, in the computation of the left-trivialized derivative of $\ls^0\vX_A$. 
We choose to compute the quantities $\ls^0\vX_A \ls^A\vv_{A,0}$ and $\ls^0\vX_A \ls^A\va_{grav}$ as a whole, as this allows us to compute the derivative of a vector instead of the derivative of a matrix. 
This is explained in the following proposition. 

\begin{proposition}
    The left-trivialized derivatives of $\ls^0\vX_A \ls^A\vv_{A,0}$ and $\ls^0\vX_A \ls^A\va_{grav}$ with respect to the transformation matrix $\ls ^A\vH_0$ are given by 
    \begin{align}
        \frac{\tp (\ls^0\vX_A \ls^A\vv_{A,0})}{\partial \ls^A\vH_0} 
        &= 
        \begin{bmatrix}
          \ls^0\vR_A \ls^A\vomega_{A,0}^\wedge \ls^A\vR_0 
          &
          \ls^0\vR_A (\ls^A\bmv_{A,0} - \ls^A\vo_0^\wedge \ls^A\vomega_{A,0} )^\wedge \ls^A\vR_0
          \\
          0_{3 \times 3} 
          &
          \ls^0\vR_A \ls^A\vomega_{A,0}^\wedge \ls^A\vR_0
        \end{bmatrix} \label{eq:partialXvwrtH} 
        \shortintertext{and} 
        \frac{\tp (\ls^0\vX_A \ls^A\va_{grav})}{\partial \ls^A\vH_0} 
        &=
        \begin{bmatrix}
          \ls^0\vR_A \ls^A\valpha_{grav}^\wedge \ls^A\vR_0 
          &
          \ls^0\vR_A (\ls^A\bma_{grav} - \ls^A\vo_0^\wedge \ls^A\valpha_{grav} )^\wedge \ls^A\vR_0
          \\
          0_{3 \times 3} 
          &
          \ls^0\vR_A \ls^A\valpha_{grav}^\wedge \ls^A\vR_0
        \end{bmatrix} \label{eq:partialXawrtH} ,
    \end{align}
    where $\ls^A\vv_{A,0} = [\ls^A\bmv_{A,0} ; \ls^A\vomega_{A,0}]$ and $\ls^A\va_{grav} = [\ls^A\bma_{grav} ; \ls^A\valpha_{grav}]$. Here, $\bmv \in \R^3$ represents translational velocity, $\vomega \in \R^3$ rotational velocity, $\bma \in \R^3$ translational acceleration and $\valpha \in \R^3$ rotational acceleration. The terms $\ls^A\vv_{A,0}$ and $\ls^A\va_{grav}$ are independent of $\ls^A\vH_0$, as they are inputs to the algorithm. 
    $\null \hfill \blacksquare$
\end{proposition}

\begin{proof}
In the following, only a proof for \eqref{eq:partialXvwrtH} will be given. 
A proof for \eqref{eq:partialXawrtH} follows straightforwardly by employing the
same techniques. 
In the following, for sake of readability, we use the abbreviations 
$\vH := \ls^A\vH_0$, 
$\vR := \ls^A\vR_0$, 
$\vo := \ls^A\vo_0$, 
$\vv := \ls^A\vv_{A,0}$, 
$\bmv := \ls^A\bmv_{A,0}$, 
$\vomega := \ls^A\vomega_{A,0}$ and 
$\vX := \ls^0\vX_A$.
Explicitly, $\vX \vv$ is given by
\begin{align}  \label{eq:Xvfinal}
    \vX \vv 
    = 
    \begin{bmatrix}
      \vR^T & -\vR^T \vo^\wedge \\
      0_{3 \times 3} & \vR^T
    \end{bmatrix}
    \begin{bmatrix}
      \bmv \\
      \vomega
    \end{bmatrix} 
    =
    \begin{bmatrix} 
      \vR^T \bmv - \vR^T \vo^\wedge \vomega \\
      \vR^T \vomega 
    \end{bmatrix}.
\end{align}
Define the vector valued functions $f_1 : \R^3 \times \SO(3) \rightarrow \R^3$ and $f_2 : \SO(3) \rightarrow \R^3$ as
\begin{align}
    \label{eq:f1}
    f_1(\vo,\vR) 
    &:
    = \vR^T \bmv - \vR^T \vo^\wedge \vomega \quad \text{and}
    \\
    \label{eq:f2}
    f_2(\vo, \vR) 
    &:=
    \vR^T \vomega ,
\end{align}
and rewrite \eqref{eq:Xvfinal} as
\begin{equation} \label{eq:Xv}
    \vX \vv 
    = 
    \vh(\vH)
    =
    f(\vo,\vR)
    =
    \begin{bmatrix}
      f_1(\vo,\vR) \\
      f_2(\vo,\vR)
    \end{bmatrix} ,
\end{equation}
where $\vh(\vH) : \SE(3) \rightarrow \R^6$ and $f(\vo,\vR) : \R^3 \times \SO(3) \rightarrow \R^6$. 
The left-trivialized derivative of $\vh(\vH)$ in a direction $\vDelta = \begin{bmatrix} \vDelta_v ; \vDelta_\omega \end{bmatrix} \in \R^6$ is 
\begin{equation} \label{eq:lefttrivderivh}
    \D \vh(\vH) \cdot T_I L_\vH \vDelta 
    =
    \D \vh(\vH) \cdot \vH \vDelta^\wedge .
\end{equation}
Since
\begin{equation}
    \vH \vDelta^\wedge 
    =
    \begin{bmatrix}
        \vR \vDelta_\omega^\wedge & \vR \vDelta_v   \\
        0_{1\times 3} & 0
    \end{bmatrix} ,
\end{equation}
we can rewrite \eqref{eq:lefttrivderivh} using \eqref{eq:f1} and \eqref{eq:f2} 
more specifically as
\begin{equation}
    \D \vh(\vH) \cdot T_I L_\vH \vDelta 
    =
    \D_1 f(\vo,\vR) \cdot \vR \vDelta_v
    + 
    \tilde{\D}_2 f(\vo,\vR) \cdot \vDelta_\omega .
\end{equation}
In matrix form, we express the left-trivialized derivative of $\vh(\vH)$ as
\begin{equation}
   \D \vh(\vH) \cdot T_I L_\vH
    = 
    \begin{bmatrix}
        \D_1 f_1(\vo,\vR) \vR 
        &
        \D_2 f_1(\vo,\vR) \cdot T_I L_\vR
        \\
        \D_1 f_2(\vo,\vR) \vR 
        &
        \D_2 f_2(\vo,\vR) \cdot T_I L_\vR
    \end{bmatrix} .
\end{equation}
From this follows immediately that
\begin{equation}
    \D_1 f_1(\vo,\vR) \vR 
    =
    \vR^T \vomega^\wedge \vR ,
\end{equation}
\begin{equation}
    \D_2 f_1(\vo,\vR) \cdot T_I L_\vR
    =
    \vR^T (\bmv - \vo^\wedge \vomega )^\wedge \vR ,
\end{equation}
\begin{equation}
    \D_1 f_2(\vo,\vR) \vR 
    =
    0_{3 \times 3} \quad \text{and}
\end{equation}
\begin{equation}
    \D_2 f_2(\vo,\vR) \cdot T_I L_\vR
    =
    \vR^T \vomega^\wedge \vR .
\end{equation}
This concludes the proof.
\end{proof}

\noindent
We are now ready to provide the algorithm  that computes the left-trivialized derivative of the extended inverse dynamics given by \eqref{eq:inversedynamics} with respect to the transformation matrix $\vH$. The algorithm is reported in Table~
\ref{tab:derivativeH}.

\begingroup
\renewcommand{\arraystretch}{1.2}
\begin{table}[H]
\caption{Left-trivialized derivatives of the extended inverse dynamics with respect to the transformation matrix $\vH$.}
\label{tab:derivativeH}
\centering
\begin{tabular}{|l|l|l|}
\hline
\multicolumn{3}{|l|}{\textbf{Inputs:} \textit{All outputs and intermediate variables of EIDAmb}}
   \\ \hline
\textbf{Line} & \textbf{Algorithm} & \textbf{Line in} \\
 & & \textbf{EIDAmb} \\\hline
1  & $\dfrac{\tp \vv_0}{\partial \vH} = \begin{bmatrix}
          \ls^0\vR_A \ls^A\vomega_{A,0}^\wedge \ls^0\vR_A^T 
          &
          \ls^0\vR_A (\ls^A\bmv_{A,0} - \ls^A\vo_0^\wedge \ls^A\vomega_{A,0} )^\wedge \ls^0\vR_A^T
          \\
          0_{3 \times 3} 
          &
          \ls^0\vR_A \ls^A\vomega_{A,0}^\wedge \ls^0\vR_A^T
        \end{bmatrix}$
   & 5
   \\[2ex]
2  & $\dfrac{\tp \va^r_0}{\partial \vH} = \begin{bmatrix}
          \ls^0\vR_A \ls^A\valpha_{grav}^\wedge \ls^0\vR_A^T 
          &
          \ls^0\vR_A (\ls^A\bma_{grav} - \ls^A\vo_0^\wedge \ls^A\valpha_{grav} )^\wedge \ls^0\vR_A^T
          \\
          0_{3 \times 3} 
          &
          \ls^0\vR_A \ls^A\valpha_{grav}^\wedge \ls^0\vR_A^T
        \end{bmatrix}$
   & 6
   \\[2ex]
3  & $\dfrac{\tp \vm_{\cB 0}}{\partial \vH} = 
     \bbM_{\cB 0} \dfrac{\tp \vv_0}{\partial \vH}$
   & 9
   \\[2ex]
4  & $\dfrac{\tp \vb^c_{\cB 0}}{\partial \vH} = 
     \bbM_{\cB 0} \dfrac{\tp \va^r_0}{\partial \vH} + 
     \dfrac{\tp \vv_0}{\partial \vH} \bts \vm_{\cB 0} + 
     \vv_0 \bts \dfrac{\tp \vm_{\cB 0}}{\partial \vH}$
   & 10
   \\[2ex]
5  & $\mathbf{for}\ i=1\ \mathbf{to}\ n_B\ \mathbf{do}$  
   & 12
   \\
6  & $\quad \dfrac{\tp \vv_i}{\partial \vH} =
     \ls^i\vX_{\lambda(i)} 
     \dfrac{\tp \vv_{\lambda(i)}}{\partial \vH}$
   & 16
   \\[2ex]
7  & $\quad \dfrac{\tp \va^r_i}{\partial \vH} =
     \ls^i\vX_{\lambda(i)} 
     \dfrac{\tp \va^r_{\lambda(i)}}{\partial \vH} + 
     \dfrac{\tp \vv_i}{\partial \vH} \times \vv_{\cJ i}$
   & 17
   \\[2ex]
8  & $\quad \dfrac{\tp \vm_{\cB i}}{\partial \vH} = 
     \bbM_{\cB i} \dfrac{\tp \vv_i}{\partial \vH}$
   & 20
   \\[2ex]
9  & $\quad \dfrac{\tp \vb^c_{\cB i}}{\partial \vH} = 
     \bbM_{\cB i} \dfrac{\tp \va^r_i}{\partial \vH} + 
     \dfrac{\tp \vv_i}{\partial \vH} \bts \vm_{\cB i} + 
     \vv_i \bts \dfrac{\tp \vm_{\cB i}}{\partial \vH}$
   & 21
   \\[2ex]
10 & $\mathbf{end}$ 
   & 23
   \\
11 & $\mathbf{for}\ i=n_B\ \mathbf{to}\ 1\ \mathbf{do}$  
   & 25
   \\
12 & $\quad \dfrac{\tp \vb^c_{\cB \lambda(i)}}{\partial \vH} =
     \dfrac{\tp \vb^c_{\cB \lambda(i)}}{\partial \vH} +
     \ls_{\lambda(i)}\vX^i 
     \dfrac{\tp \vb^c_{\cB i}}{\partial \vH}$
   & 26
   \\[2ex]
13 & $\mathbf{end}$ 
   & 28
   \\
14 & $\mathbf{for}\ i=1\ \mathbf{to}\ n_B\ \mathbf{do}$  
   & 29
   \\
15 & $\quad \dfrac{\tp \vtau_i}{\partial \vH} = 
     \vGamma_{\cJ i}^T \dfrac{\tp \vb^c_{\cB i}}{\partial \vH}$
   & 31
   \\[2ex]
16 & $\mathbf{end}$ 
   & 39
   \\
17 & $\dfrac{\tp \bar{\vtau}_b}{\partial \vH} = 
     \dfrac{\tp \vb^c_{\cB 0}}{\partial \vH}$
   & 40
   \\[2ex] \hline
\multicolumn{3}{|l|}{\textbf{Outputs:} $\D_1 \overline{ID} \circ \D L_H (I) = \tp \bar{\vtau} / \partial \vH = [\tp \vtau / \partial \vH ; \tp \bar{\vtau}_b / \partial \vH]$}
   \\ \hline
\end{tabular}
\end{table}
\endgroup

\noindent
\textit{Remarks:} 
\begin{itemize}
    \item Line 1 computes the left trivialized derivative $\tp ( \ls^0\vX_A \ls^A\vv_{A,0} ) / \partial \ls^A\vH_0$ in \eqref{eq:partialXvwrtH}.
    \item Line 2 computes the left trivialized derivative $\tp ( \ls^0\vX_A \ls^A\va_{grav} ) / \partial \ls^A\vH_0$ in \eqref{eq:partialXawrtH}.
    \item Line 17 computes the left-trivialized derivative of the non-physical control inputs with respect to the transformation matrix $\vH$ in \eqref{eq:ltlinverseH}.
\end{itemize}

\noindent\textbf{Second algorithm ($\vs$).} The second algorithm computes the derivatives of the extended inverse dynamics \eqref{eq:inversedynamics} with respect to the generalized position vector $\vs$. 
The right column shows which line in the EIDAmb each equation origins from. 
The function jcalcderiv(jtype(i),$\vs_i$) acts as in for the EIDAmb treated in the previous subsection as a look-up table, from where joint-type-specific functions are retrieved. 
These joint-type-specific functions compute the derivative of the velocity transformation matrix $\ls^i\vX_{\lambda(i|i}$ with respect to the generalized position $\vs_i$. 
The explicit joint-type-specific functions can be found in \cite[Appendix B.3]{bos2019} or can be straightforwardly derived using, e.g., \cite[Section 4.4]{featherstone2008rigid}. 

\begingroup
\renewcommand{\arraystretch}{1.2}
\begin{table}[H]
\caption{Derivatives of the extended inverse dynamics with respect to the generalized position vector $\vs$.}
\label{tab:derivatives}
\centering
\begin{tabular}{|p{0.8cm}|p{10cm}|p{1.6cm}|}
\hline
\multicolumn{3}{|l|}{\textbf{Inputs:} \textit{All outputs and intermediate variables of EIDAmb}}
\\ \hline
\textbf{Line} & \textbf{Algorithm} & \textbf{Line in} \\
 & & \textbf{EIDAmb} \\\hline
1  & $\mathbf{for}\ i=1\ \mathbf{to}\ n_B\ \mathbf{do}$
   & 12
   \\
2  & $\quad \dfrac{\partial \ls^i\vX_{\lambda(i)|i}}{\partial \vs_i} =$  jcalcderiv(jtype($i$)$, \vs_i$) 
   & 13
   \\[2ex]
3  & $\quad \dfrac{\partial \ls^i\vX_{\lambda(i)}}{\partial \vs_i} = 
     \dfrac{\partial \ls^i\vX_{\lambda(i)|i}}{\partial \vs_i}
     \ls^{\lambda(i)|i}\vX_{\lambda(i)}$
   & 15
   \\[2ex]
4  & $\quad \dfrac{\partial \vv_i}{\partial \vs} = \ls^i\vX_{\lambda(i)} 
     \dfrac{\partial \vv_{\lambda(i)}}{\partial \vs}$
   & 16
   \\[2ex]
5  & $\quad \dfrac{\partial \vv_i}{\partial \vs_i} = 
     \dfrac{\partial \vv_i}{\partial \vs_i} + \dfrac{\partial 
     \ls^i\vX_{\lambda(i)}}{\partial \vs_i} \vv_{\lambda(i)}$
   & 16
   \\[2ex]
6  & $\quad \dfrac{\partial \va^r_i}{\partial \vs} = \ls^i\vX_{\lambda(i)} 
     \dfrac{\partial \va^r_{\lambda(i)}}{\partial \vs} + 
     \dfrac{\partial \vv_i}{\partial \vs} \times \vv_{\cJ i}$
   & 17
   \\[2ex]
7  & $\quad \dfrac{\partial \va^r_i}{\partial \vs_i} = \dfrac{\partial
     \va^r_i}{\partial \vs_i} + 
     \dfrac{\partial \ls^i\vX_{\lambda(i)}}{\partial \vs_i} 
     \va^r_{\lambda(i)}$
   & 17
   \\[2ex]
8  & $\quad \dfrac{\partial \va^{vp}_i}{\partial \vs} = \ls^i\vX_{\lambda(i)} 
     \dfrac{\partial \va^{vp}_{\lambda(i)}}{\partial \vs} + 
     \dfrac{\partial \vv_i}{\partial \vs} \times \vv_{\cJ i}$
   & 18
   \\[2ex]
9  & $\quad \dfrac{\partial \va^{vp}_i}{\partial \vs_i} = \dfrac{\partial
     \va^{vp}_i}{\partial \vs_i} + 
     \dfrac{\partial \ls^i\vX_{\lambda(i)}}{\partial \vs_i} 
     \va^{vp}_{\lambda(i)}$
   & 18
   \\[2ex]
10 & $\quad \dfrac{\partial \vm_{\cB i}}{\partial \vs} = 
      \bbM_{\cB i} \dfrac{\partial \vv_i}{\partial \vs}$
   & 20
   \\[2ex]   
11 & $\quad \dfrac{\partial \vb^c_{\cB i}}{\partial \vs} = \bbM_{\cB i} 
     \dfrac{\partial \va^r_i}{\partial \vs} + 
     \dfrac{\partial \vv_i}{\partial \vs} \bts \vm_{\cB i} +
     \vv_i \bts \dfrac{\partial \vm_{\cB i}}{\partial \vs}$
   & 21
   \\[2ex]   
12 & $\quad \dfrac{\partial \vb^{vp}_{\cB i}}{\partial \vs} = \bbM_{\cB i} 
     \dfrac{\partial \va^{vp}_i}{\partial \vs} + 
     \dfrac{\partial \vv_i}{\partial \vs} \bts \vm_{\cB i} +
     \vv_i \bts \dfrac{\partial \vm_{\cB i}}{\partial \vs}$
   & 22
   \\[2ex]
13 & $\mathbf{end}$ 
   & 23
   \\
\hline
\multicolumn{3}{|c|}{\textit{continued on the next page.}}
   \\\hline
\end{tabular}
\end{table}
\endgroup

\begingroup
\renewcommand{\arraystretch}{1.2}
\begin{table}[H]
\centering
\begin{tabular}{|p{0.8cm}|p{10cm}|p{1.6cm}|}
\hline
\multicolumn{3}{|c|}{\textit{continued from the previous page.}}
   \\
\hline   
14 & $\mathbf{for}\ i=n_B\ \mathbf{to}\ 1\ \mathbf{do}$
   & 24
   \\
15 & $\quad \mathbf{for}\ k=1\ \mathbf{to}\ n_B\ \mathbf{do}$
   & 25
   \\ 
16 & $\qquad \dfrac{\partial \bbM^c_{\cB \lambda(i)}}{\partial \vs_k} = 
     \dfrac{\partial \bbM^c_{\cB \lambda(i)}}{\partial \vs_k} +
     \ls_{\lambda(i)}\vX^i \dfrac{\partial \bbM^c_{\cB i}}{\partial \vs_k}
     \ls^i\vX_{\lambda(i)}$
   & 25
   \\[2ex]
17 & $\quad \mathbf{end}$
   & 25
   \\
18 & $\quad \dfrac{\partial \bbM^c_{\cB \lambda(i)}}{\partial \vs_i} = 
     \dfrac{\partial \bbM^c_{\cB \lambda(i)}}{\partial \vs_i} +
     \dfrac{\partial \ls_{\lambda(i)}\vX^i}{\partial \vs_i} 
     \bbM^c_{\cB i} \ls^i\vX_{\lambda(i)} + \ls_{\lambda(i)}\vX^i \:
     \bbM^c_{\cB i} \dfrac{\partial \ls^i\vX_{\lambda(i)}}{\partial \vs_i}$
   & 25
   \\[2ex]
19 & $\quad \dfrac{\partial \vb^c_{\cB \lambda(i)}}{\partial \vs} =
     \dfrac{\partial \vb^c_{\cB \lambda(i)}}{\partial \vs} + 
     \ls_{\lambda(i)}\vX^i \dfrac{\partial \vb^c_{\cB i}}{\partial \vs}$
   & 26
   \\[2ex]
20 & $\quad \dfrac{\partial \vb^c_{\cB \lambda(i)}}{\partial \vs_i} =
     \dfrac{\partial \vb^c_{\cB \lambda(i)}}{\partial \vs_i} + 
     \dfrac{\partial \ls_{\lambda(i)}\vX^i}{\partial \vs_i} \vb^c_{\cB i}$
   & 26
   \\[2ex] 
21 & $\quad \dfrac{\partial \vb^{vp}_{\cB \lambda(i)}}{\partial \vs} =
     \dfrac{\partial \vb^{vp}_{\cB \lambda(i)}}{\partial \vs} + 
     \ls_{\lambda(i)}\vX^i \dfrac{\partial \vb^{vp}_{\cB i}}{\partial \vs}$
   & 27
   \\[2ex] 
22 & $\quad \dfrac{\partial \vb^{vp}_{\cB \lambda(i)}}{\partial \vs_i} =
     \dfrac{\partial \vb^{vp}_{\cB \lambda(i)}}{\partial \vs_i} + 
     \dfrac{\partial \ls_{\lambda(i)}\vX^i}{\partial \vs_i} \vb^{vp}_{\cB i}$
   & 27
   \\[2ex]
23 & $\mathbf{end}$ 
   & 28
   \\
24 & $\mathbf{for}\ k=1\ \mathbf{to}\ n_B\ \mathbf{do}$
   & -
   \\
25 & $\quad \dfrac{\partial \bbM^c_{\cB 0} \va_0}{\partial \vs_k} = 
     \dfrac{\partial \bbM^c_{\cB 0}}{\partial \vs_k} \va_0$
   & -
   \\[2ex]
26 & $\mathbf{end}$
   & -
   \\
27 & $\mathbf{for}\ i=1\ \mathbf{to}\ n_B\ \mathbf{do}$  
   & 29
   \\
28 & $\quad \dfrac{\partial \ls^i\va_{A,0}}{\partial \vs} = 
     \ls^i\vX_{\lambda(i)}
     \dfrac{\partial \ls^{\lambda(i)}\va_{A,0}}{\partial \vs}$
   & 30
   \\[2ex]
29 & $\quad \dfrac{\partial \ls^i\va_{A,0}}{\partial \vs_i} = 
     \dfrac{\partial \ls^i\va_{A,0}}{\partial \vs_i} +
     \dfrac{\partial \ls^i\vX_{\lambda(i)}}{\partial \vs_i} 
     \ls^{\lambda(i)}\va_{A,0}$
   & 30
   \\[2ex]
30 & $\quad \mathbf{for}\ k=1\ \mathbf{to}\ n_B\ \mathbf{do}$
   & -
   \\
31 & $\qquad \dfrac{\partial \bbM^c_{\cB i} \ls^i\va_{A,0}}{\partial \vs_k} =
     \dfrac{\partial \bbM^c_{\cB i}}{\partial \vs_k} \ls^i\va_{A,0}$
   & -
   \\[2ex]
32 & $\qquad \dfrac{\partial \ls_i\vF_{\cB i}}{\partial \vs_k} = 
     \dfrac{\partial \bbM^c_{\cB i}}{\partial \vs_k} \vGamma_{\cJ i}$
   & 32
   \\[2ex]
33 & $\quad \mathbf{end}$
   & -
   \\
34 & $\quad \dfrac{\partial \vtau_i}{\partial \vs} = 
     \vGamma_{\cJ i}^T
     \Big( \dfrac{\partial \bbM^c_{\cB i}}{\partial \vs}
     \ls^i\va_{A,0} + 
     \bbM^c_{\cB i} \dfrac{\partial \ls^i\va_{A,0}}{\partial \vs} +
     \dfrac{\partial \vb^c_{\cB i}}{\partial \vs} \Big)$
   & 32
   \\[2ex]
35 & $\quad j = i$
   & 33
   \\
36 & $\quad \mathbf{while}\ \lambda(j) > 0$
   & 34
   \\
37 & $\qquad \mathbf{for}\ k=1\ \mathbf{to}\ n_B\ \mathbf{do}$
   & 35
   \\
38 & $\quad \qquad \dfrac{\partial \ls_{\lambda(j)}\vF_{\cB i}}{\partial \vs_k} = 
     \ls_{\lambda(j)}\vX^j \dfrac{\partial \ls_j\vF_{\cB i}}{\partial \vs_k}$
   & 35
   \\[2ex]
39 & $\qquad \mathbf{end}$
   & 35
   \\
\hline
\multicolumn{3}{|c|}{\textit{continued on the next page.}}
   \\\hline
\end{tabular}
\end{table}
\endgroup

\begingroup
\renewcommand{\arraystretch}{1.2}
\begin{table}[H]
\centering
\begin{tabular}{|p{0.8cm}|p{10cm}|p{1.6cm}|}
\hline
\multicolumn{3}{|c|}{\textit{continued from the previous page.}}
   \\\hline
40 & $\qquad \dfrac{\partial \ls_{\lambda(j)}\vF_{\cB i}}{\partial \vs_j} = 
      \dfrac{\partial \ls_{\lambda(j)}\vF_{\cB i}}{\partial \vs_j} +
     \dfrac{\partial \ls_{\lambda(j)}\vX^j}{\partial \vs_j} \ls_j\vF_{\cB i}$
   & 35
   \\[2ex]
41 & $\qquad j = \lambda(j)$
   & 36
   \\
42 & $\quad \mathbf{end}$
   & 37
   \\
43 & $\quad \mathbf{for}\ k=1\ \mathbf{to}\ n_B\ \mathbf{do}$
   & 38
   \\
44 & $\qquad \dfrac{\partial \ls_0\vF_{\cB i}}{\partial \vs_k} = 
     \ls_0\vX^j \dfrac{\partial \ls_j\vF_{\cB i}}{\partial \vs_k}$
   & 38
   \\[2ex]
45 & $\quad \mathbf{end}$
   & 38
   \\   
46 & $\quad \dfrac{\partial \ls_0\vF_{\cB i}}{\partial \vs_j} = 
      \dfrac{\partial \ls_0\vF_{\cB i}}{\partial \vs_j} +
     \dfrac{\partial \ls_0\vX^j}{\partial \vs_j} \ls_j\vF_{\cB i}$
   & 38
   \\[2ex]
47 & $\quad \dfrac{\partial \ls_0\vF \dot{\vr}}{\partial \vs} = 
      \dfrac{\partial \ls_0\vF \dot{\vr}}{\partial \vs} + 
      \dfrac{\partial \ls_0\vF_{\cB i} }{\partial \vs} \dot{\vr}$
   & -
   \\[2ex]
48 & $\mathbf{end}$ 
   & 39
   \\ 
49 & $\dfrac{\partial \bar{\vtau}_b}{\partial \vs} = 
     \dfrac{\partial \bbM^c_{\cB 0} \va_0 }{\partial \vs} +
     \dfrac{\partial \ls_0\vF \dot{\vr}}{\partial \vs} +
     \dfrac{\partial \vb^{vp}_{\cB 0}}{\partial \vs}$
   & 40
   \\[2ex] \hline
\multicolumn{3}{|l|}{\textbf{Outputs:} $\D_2 \overline{ID} = \partial \bar{\vtau} / \partial \vs = [ \partial \bar{\vtau}_b / \partial \vs ; \partial \vtau / \partial \vs ]$}
   \\ \hline
\end{tabular}
\end{table}
\endgroup

\noindent
\textit{Remark:} Line 49 computes the derivative of the non-physical control input with respect to the generalized position vector $\vs$ in \eqref{eq:ltlinverses}.

~\\
\noindent\textbf{Third algorithm ($\vv$).} The third algorithm computes the derivatives of the extended inverse dynamics \eqref{eq:inversedynamics} with respect to the moving-base velocity $\vv$. 
The right column shows which line in the EIDAmb each equation origins from. 

\begingroup
\renewcommand{\arraystretch}{1.2}
\begin{table}[H]
\caption{Derivatives of the extended inverse dynamics with respect to the moving-base velocity $\vv$.}
\label{tab:derivativev}
\centering
\begin{tabular}{|p{0.8cm}|p{7cm}|p{1.6cm}|}
\hline
\multicolumn{3}{|l|}{\textbf{Inputs:} \textit{All outputs and intermediate variables of EIDAmb}}
\\ \hline
\textbf{Line} & \textbf{Algorithm} & \textbf{Line in} \\
 & & \textbf{EIDAmb} \\\hline
1  & $\dfrac{\partial \vv_0}{\partial \vv} = I_6$
   & 5
   \\[2ex]
2  & $\dfrac{\partial \vm_{\cB 0}}{\partial \vv} = \bbM_{\cB 0}$
   & 9
   \\[2ex]
3  & $\dfrac{\partial \vb^c_{\cB 0}}{\partial \vv} = 
     \dfrac{\partial \vv_0}{\partial \vv} \bts 
     \vm_{\cB 0} + \vv_0 \bts \dfrac{\partial \vm_{\cB 0}}{\partial \vv}$
   & 10
   \\[2ex]
4  & $\mathbf{for}\ i=1\ \mathbf{to}\ n_B\ \mathbf{do}$  
   & 12
   \\
5  & $\quad \dfrac{\partial \vv_i}{\partial \vv} = \ls^i\vX_{\lambda(i)} 
     \dfrac{\partial \vv_{\lambda(i)}}{\partial \vv}$
   & 16
   \\[2ex]
6  & $\quad \dfrac{\partial \va^r_i}{\partial \vv} = 
     \ls^i\vX_{\lambda(i)} 
     \dfrac{\partial \va^r_{\lambda(i)}}{\partial \vv} + 
     \dfrac{\partial \vv_i}{\partial \vv} \times \vv_{\cJ i}$
   & 17
   \\[2ex]    
\hline
\multicolumn{3}{|c|}{\textit{continued on the next page.}}
   \\\hline
\end{tabular}
\end{table}
\endgroup

\begingroup
\renewcommand{\arraystretch}{1.2}
\begin{table}[H]
\centering
\begin{tabular}{|p{0.8cm}|p{7cm}|p{1.6cm}|}
\hline
\multicolumn{3}{|c|}{\textit{continued from the previous page.}}
   \\
\hline
7  & $\quad \dfrac{\partial \vm_{\cB i}}{\partial \vv} = 
      \bbM_{\cB i} \dfrac{\partial \vv_i}{\partial \vv}$
   & 20
   \\[2ex] 
8  & $\quad \dfrac{\partial \vb^c_{\cB i}}{\partial \vv} = \bbM_{\cB i} 
     \dfrac{\partial \va^r_i}{\partial \vv} + 
     \dfrac{\partial \vv_i}{\partial \vv} \bts \vm_{\cB i} +
     \vv_i \bts \dfrac{\partial \vm_{\cB i}}{\partial \vv}$
   & 21
   \\[2ex] 
9  & $\mathbf{end}$ 
   & 23
   \\ 
10 & $\mathbf{for}\ i=n_B\ \mathbf{to}\ 1\ \mathbf{do}$
   & 24
   \\
11 & $\quad \dfrac{\partial \vb^c_{\cB \lambda(i)}}{\partial \vv} =
     \dfrac{\partial \vb^c_{\cB \lambda(i)}}{\partial \vv} + 
     \ls_{\lambda(i)}\vX^i \dfrac{\partial \vb^c_{\cB i}}{\partial \vv}$
   & 26
   \\[2ex]
12 & $\mathbf{end}$ 
   & 28
   \\
13 & $\mathbf{for}\ i=1\ \mathbf{to}\ n_B\ \mathbf{do}$  
   & 29
   \\
14 & $\quad \dfrac{\partial \vtau_i}{\partial \vv} = 
     \vGamma_{\cJ i}^T \dfrac{\partial \vb^c_{\cB i}}{\partial \vv}$
   & 31
   \\[2ex]
15 & $\mathbf{end}$ 
   & 39
   \\ 
16 & $\dfrac{\partial \bar{\vtau}_b}{\partial \vv} = 
     \dfrac{\partial \vb^c_{\cB 0}}{\partial \vv}$
   & 40
   \\[2ex] \hline
\multicolumn{3}{|l|}{\textbf{Outputs:} $\D_3 \overline{ID} = \partial \bar{\vtau} / \partial \vv = [ \partial \bar{\vtau}_b / \partial \vv ; \partial \vtau / \partial \vv ]$}
   \\ \hline
\end{tabular}
\end{table}
\endgroup

\noindent
\textit{Remark:} Line 16 computes the derivative of the non-physical control input with respect to the moving-base velocity $\vv$ in \eqref{eq:ltlinversev}.

~\\
\noindent\textbf{Fourth algorithm ($\vr$).} The fourth algorithm computes the derivatives of the extended inverse dynamics \eqref{eq:inversedynamics} with respect to the generalized velocity vector $\vr$. The right column shows which line in the EIDAmb each equation origins from. 

\begingroup
\renewcommand{\arraystretch}{1.2}
\begin{table}[H]
\caption{Derivatives of the extended inverse dynamics w.r.t. the generalized velocity vector $\vr$.}
\label{tab:derivativer}
\centering
\begin{tabular}{|p{0.8cm}|p{7cm}|p{1.6cm}|}
\hline
\multicolumn{3}{|l|}{\textbf{Inputs:} \textit{All outputs and intermediate variables of EIDAmb}}
\\ \hline
\textbf{Line} & \textbf{Algorithm} & \textbf{Line in} \\
 & & \textbf{EIDAmb} \\\hline
1  & $\mathbf{for}\ i=1\ \mathbf{to}\ n_B\ \mathbf{do}$  
   & 12
   \\
2  & $\quad \dfrac{\partial \vv_{\cJ i}}{\partial \vr_i} = \vGamma_{\cJ i}$
   & 14
   \\[2ex]
3  & $\quad \dfrac{\partial \vv_i}{\partial \vr} = \ls^i\vX_{\lambda(i)} 
     \dfrac{\partial \vv_{\lambda(i)}}{\partial \vr}$
   & 16
   \\[2ex]
4  & $\quad \dfrac{\partial \vv_i}{\partial \vr_i} = 
     \dfrac{\partial \vv_i}{\partial \vr_i} + 
     \dfrac{\partial \vv_{\cJ i}}{\partial \vr_i}$
   & 16
   \\[2ex]
5  & $\quad \dfrac{\partial \va^r_i}{\partial \vr} = \ls^i\vX_{\lambda(i)} 
     \dfrac{\partial \va^r_{\lambda(i)}}{\partial \vr} + 
     \dfrac{\partial \vv_i}{\partial \vr} \times \vv_{\cJ i}$
   & 17
   \\[2ex] 
6  & $\quad \dfrac{\partial \va^r_i}{\partial \vr_i} = \dfrac{\partial
     \va^r_i}{\partial \vr_i} + \vv_i \times 
     \dfrac{\partial \vv_{\cJ i}}{\partial \vr_i}$
   & 17
   \\[2ex]    
\hline
\multicolumn{3}{|c|}{\textit{continued on the next page.}}
   \\\hline
\end{tabular}
\end{table}
\endgroup

\begingroup
\renewcommand{\arraystretch}{1.2}
\begin{table}[H]
\centering
\begin{tabular}{|p{0.8cm}|p{7cm}|p{1.6cm}|}
\hline
\multicolumn{3}{|c|}{\textit{continued from the previous page.}}
   \\
\hline
7  & $\quad \dfrac{\partial \vm_{\cB i}}{\partial \vr} = 
      \bbM_{\cB i} \dfrac{\partial \vv_i}{\partial \vr}$
   & 20
   \\[2ex] 
8  & $\quad \dfrac{\partial \vb^c_{\cB i}}{\partial \vr} = \bbM_{\cB i} 
     \dfrac{\partial \va^r_i}{\partial \vr} + 
     \dfrac{\partial \vv_i}{\partial \vr} \bts \vm_{\cB i} +
     \vv_i \bts \dfrac{\partial \vm_{\cB i}}{\partial \vr}$
   & 21
   \\[2ex]
9  & $\mathbf{end}$ 
   & 23
   \\
10 & $\mathbf{for}\ i=n_B\ \mathbf{to}\ 1\ \mathbf{do}$
   & 24
   \\
11 & $\quad \dfrac{\partial \vb^c_{\cB \lambda(i)}}{\partial \vr} =
     \dfrac{\partial \vb^c_{\cB \lambda(i)}}{\partial \vr} + 
     \ls_{\lambda(i)}\vX^i \dfrac{\partial \vb^c_{\cB i}}{\partial \vr}$
   & 26
   \\[2ex] 
12 & $\mathbf{end}$ 
   & 28
   \\
13 & $\mathbf{for}\ i=1\ \mathbf{to}\ n_B\ \mathbf{do}$  
   & 29
   \\
14 & $\quad \dfrac{\partial \vtau_i}{\partial \vr} = 
     \vGamma_{\cJ i}^T \dfrac{\partial \vb^c_{\cB i}}{\partial \vr}$
   & 31
   \\[2ex]
15 & $\mathbf{end}$ 
   & 39
   \\ 
16 & $\dfrac{\partial \bar{\vtau}_b}{\partial \vr} = 
     \dfrac{\partial \vb^c_{\cB 0}}{\partial \vr}$
   & 40
   \\[2ex]
   \hline
\multicolumn{3}{|l|}{\textbf{Outputs:} $\D_4 \overline{ID} = \partial \bar{\vtau} / \partial \vr = [ \partial \bar{\vtau}_b / \partial \vr ; \partial \vtau / \partial \vr ]$} 
\\ \hline
\end{tabular}
\end{table}
\endgroup

\noindent
\textit{Remark:} Line 16 computes the derivative of the non-physical control input with respect to the generalized velocity vector $\vr$ in \eqref{eq:ltlinverser}.

\subsection{Inverse mass matrix algorithm for moving-base systems}
\label{sec:IMMAmb}

In \cite{carpentier2018analyticalinverse}, a computational efficient algorithm --we will refer to this algorithm, in the following, as the Inverse Mass Matrix Algorithm (IMMA)-- is introduced to compute the inverse of the mass matrix for fixed-base systems. The IMMA is based on the Articulated Body Algorithm (ABA) which computes the joint accelerations as a function of the joint configuration, velocity, and applied torque. In this section, we propose an extended version of the IMMA that computes the mass matrix for moving-base systems. We dub this algorithm the Inverse Mass Matrix Algorithm for moving-base systems (IMMAmb). 

The forward dynamics for moving-base systems given in \eqref{eq:forwarddynamics} can be rewritten as
\begin{equation}
    FD(\vH,\vs,\vv,\vr,\vtau) 
    :=
    \begin{bmatrix} 
        \dot{\vv} \\ 
        \dot{\vr} 
    \end{bmatrix} 
    =
    \begin{bmatrix} 
        \vM_{11} & \vM_{12} \\ 
        \vM_{21} & \vM_{22} 
    \end{bmatrix}^{-1} 
    \begin{bmatrix} 
        - \vh_1 \\ 
        \vtau - \vh_2 
    \end{bmatrix} .
\end{equation}
The key idea of the IMMAmb is to evaluate the forward dynamics for $n_J$ appropriates values of the torque vector $\vtau$ (namely, the $n_J$ basis vectors $[1, 0, \dots, 0]$, $[0, 1, 0, \dots, 0]$, $\dots$, and $[0, \dots, 0, 1]$) while discarding the computation of the Coriolis and gravity terms $\vh_1$ and $\vh_2$. While this could be done in a for loop, calling the forward dynamics several times, the key idea we use (taken from IMMA) is that such an evaluation can be done in parallel so to evaluate intermediate computations only once. The result of this evaluation is the tall matrix corresponding to the blocks $\vM^{-1}_{12}$ and $\vM^{-1}_{22}$ of the inverse mass matrix $\vM^-1$ (more precisely, the IMMAmb only computes the upper diagonal part of the matrix $\vM^{-1}_{22}$). The block $\vM^{-1}_{21}$ can be computed by symmetry from  $\vM^{-1}_{12}$ and $\vM^{-1}_{11}$ with a $6\times 6$ symmetric matrix inversion.
The whole IMMAmb algorithm
is reported in Table~\ref{tab:IMMAmb}.
In the algorithm, the subtree of $i$ is defined as all bodies that are supported by joint $i$, including body $i$ itself, and $subtree(i)+6$ indicates an addition of $6$ to all elements of the set $subtree(i)$. For a given matrix $\vM$, we use Python programming language notation $\vM[i,i\!:]$ to express all columns of $\vM[i,:]$ from column $i$ up to and including the last column. Similarly, with $\vM[:,s]$, where $s$ a set of indexes, will denote the matrix resulting from the selection of the columns of $\vM$ included in $s$.

\begingroup
\renewcommand{\arraystretch}{1.2}
\begin{table}[H]
\caption{Inverse Mass Matrix Algorithm for moving-base systems}
\label{tab:IMMAmb}
\centering
\begin{tabular}{|l|l|}
\hline
\textbf{Inputs} & model, $\vs$  \\ \hline
\textbf{Line} & \textbf{IMMAmb}  \\ \hline
1  & $\mathbf{for}\ i=1\ \mathbf{to}\ n_B\ \mathbf{do}$  
   \\
2  & $\quad [ \ls^i\vX_{\lambda(i)|i}, \vGamma_{\cJ i} ] = $
     jcalc(jtype($i$)$, \vs_i$) 
   \\
3  & $\quad \ls^i\vX_{\lambda(i)} = \ls^i\vX_{\lambda(i)|i} 
     \ls^{\lambda(i)|i}\vX_{\lambda(i)}$ 
   \\
4  & $\quad \bbM_{\cB i}^A = \bbM_{\cB i}$
   \\
5  & $\mathbf{end}$
   \\
6  & $\mathbf{for}\ i=n_B\ \mathbf{to}\ 1\ \mathbf{do}$  
   \\
7  & $\quad \vU_{\cB i} = \bbM_{\cB i}^A \vGamma_{\cJ i}$
   \\
8  & $\quad \vD_{\cB i} = \vGamma_{\cJ i}^T \vU_{\cB i}$
   \\
9  & $\quad \vM^{inv}[i\!+\!6,i\!+\!6] = \vD_{\cB i}^{-1}$
   \\
10 & $\quad \vM^{inv}[i\!+\!6,subtree(i)\!+\!6] =
     \vM^{inv}[i\!+\!6,subtree(i)\!+\!6]$
   \\
   & $\phantom{\quad \vM^{inv}[i\!+\!6,subtree(i)\!+\!6] =} 
     - \vD_{\cB i}^{-1} \vGamma_{\cJ i}^T
     \mathcal{F}_i[:,subtree(i)\!+\!6]$
   \\
11 & $\quad \mathcal{F}_{\lambda(i)}[:,subtree(i)\!+\!6] =
     \mathcal{F}_{\lambda(i)}[:,subtree(i)\!+\!6]$
   \\
   & $\phantom{\quad \mathcal{F}_{\lambda(i)}[:,subtree(i)\!+\!6] =} 
     + \ls_{\lambda(i)}\vX^i
     \big( \mathcal{F}_i[:,subtree(i)\!+\!6]$
   \\
   & $\phantom{\quad \mathcal{F}_{\lambda(i)}[:,subtree(i)\!+\!6] =} 
     + \vU_{\cB i} \vM^{inv}[i\!+\!6,subtree(i)\!+\!6] \big)$
   \\
12 & $\quad \bbM^a_{\cB i} = \bbM^A_{\cB i} - \vU_{\cB i} 
     \vD_{\cB i}^{-1} \vU_{\cB i}^T$
   \\
13 & $\quad \bbM^A_{\cB \lambda(i)} = \bbM^A_{\cB \lambda(i)}  +
     \ls_{\lambda(i)}\vX^i \bbM^a_{\cB i} 
     \ls^i\vX_{\lambda(i)}$
   \\
14 & $\mathbf{end}$ 
   \\ 
15 & $\mathcal{P}_0[:,7\!:] = - (\bbM^A_{\cB 0})^{-1} \mathcal{F}_0[:,7\!:]$
   \\
16 & $\mathbf{for}\ i=1\ \mathbf{to}\ n_B\ \mathbf{do}$  
   \\
17 & $\quad \vM^{inv}[i\!+\!6,i\!+\!6\!:] = \vM^{inv}[i\!+\!6,i\!+\!6\!:] - 
     \vD_{\cB i}^{-1} \vU_{\cB i}^T \ls^i\vX_{\lambda(i)}
     \mathcal{P}_{\lambda(i)}[:,i\!+\!6\!:]$
   \\
18 & $\quad \mathcal{P}_i[:,i\!+\!6\!:] = \vGamma_{\cJ i} \vM^{inv}[i\!+\!6,i\!+\!6\!:] 
     + \ls^i\vX_{\lambda(i)} \mathcal{P}_{\lambda(i)}[i:i\!+\!6\!:]$
   \\
19 & $\mathbf{end}$ 
   \\   
20 & $\vM^{inv}[1\!:\!6,7\!:] = \mathcal{P}_0[:,7\!:]$
   \\
21 & $\mathbf{for}\ i=1\ \mathbf{to}\ n_B\ \mathbf{do}$  
   \\
22 & $\quad \mathbf{for}\ j=i\ \mathbf{to}\ n_B\ \mathbf{do}$  
   \\
23 & $\qquad \vM^{inv}[j\!+\!6,i\!+\!6] = \vM^{inv}[i\!+\!6,j\!+\!6]$
   \\
24 & $\quad \mathbf{end}$ 
   \\
25 & $\mathbf{end}$ 
   \\ 
26 & $\vM^{inv}[1\!:\!6,1\!:\!6] = (\bbM^A_{\cB 0})^{-1}$
   \\ \hline   
\textbf{Outputs} & $\vM^{inv}$  \\ \hline
\end{tabular}
\end{table}
\endgroup

\section{Numerical Validation}
\label{chp:numericalverification}
In this section, we provide numerical evidence 
of the correctness of the mathematical expressions presented in Proposition \ref{proposition:inverserelation} and the corresponding new recursive algorithms presented in Section \ref{chp:algorithmicaspects}. Such an evidence is provided
comparing the left-trivialized geometric linearization 
obtained via these algorithms with the same linearization
obtained by brute-force (geometric) finite difference.

A representative moving-base multibody system has been created, consisting of a multibody system with tree-topology and prismatic and revolute joints. By construction, the whole framework must provide correct results for all moving-base multibody systems as long as they satisfy Assumption \ref{assumption:joints}, put forward in the introduction, that all joints are conventional and modeled as 1-DoF joints.
Three criteria to build the test system have been selected:
\begin{enumerate}
    \item The test system must contain each type of 1-DoF joint (revolute, prismatic, helical), with at least one for spatial directions ($x,y,z$).
    \item The test system must contain at least one branch, since the algorithms contain several lines of code that only apply to branched systems. 
    \item The test system must contain at least one branch directly at the base (so body $0$ must have multiple children), since the algorithms contains a few lines of code that only apply to branches at the base.
\end{enumerate}
The resulting test system, for which a 3D visualization is provided in Figure~\ref{fig:model}, has been thus chosen to have nine joints, so to have three joint types oriented in three possible directions.
Furthermore, the system has a kinematic branch directly at the base (as well as further branches on child links). 
A 3D visualization that illustrates our designed model is depicted in Figure~\ref{fig:model}(a), where the yellow ball represents the \emph{moving base}, the red bodies are preceded by \emph{prismatic joint}, the green bodies by a \emph{revolute joint}, and the blue bodies by a \emph{helical joint}. 
Figure~\ref{fig:model}(b) shows the system in a state where all joints positions are equal to $0$, and the right image shows the system in a state where all joint positions are equal to $0.3$, to visualize the movement of the joints.

\begin{figure}[ht]
  \centering
  \begin{tabular}{cc}
  \includegraphics[scale=0.3]{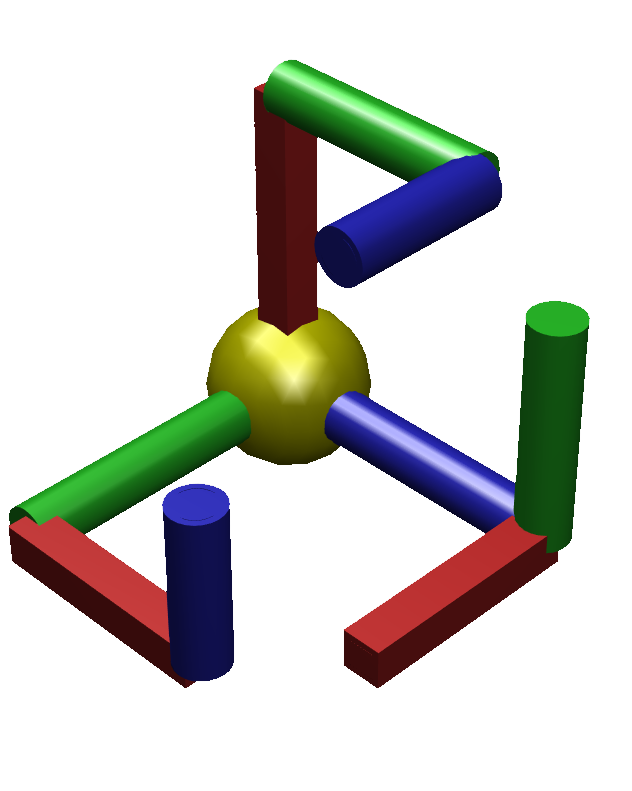} 
  &
  \includegraphics[scale=0.295]{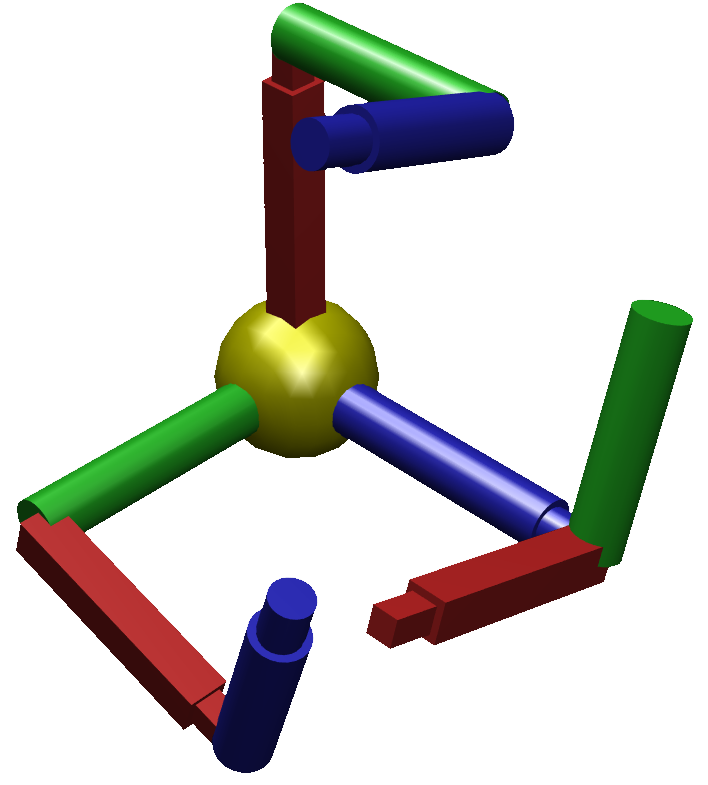}
  \\
  (a) & (b)
  \end{tabular}
  \caption{The moving-base multibody system topology selected for numerical validation; (a) reference configuration, showing a branched system with a moving base (yellow) with revolute (green), prismatic (red), and helical (blue) joints; (b) perturbed configuration, where all joint position are equal to $0.3$ (representing meters or radians, depending on the joint type).}
  \label{fig:model}
\end{figure}

\noindent
We performed a randomized test using the test model. A total of 100 configurations $(\ls^A\vH_0, \vs )$, velocities $(\ls^A\vv_{A,0}, \vr)$,
and joint accelerations $\dot{\vr}$ are
randomly selected (uniform distributions over the interval $[0,1]$, employing exponential map for generating orientation) and the corresponding moving base acceleration $\ls^0\dot{\vv}_{A,0}$ and joint torque $\vtau$ are computed through the RNEAmb for consistency (recall Definition~\ref{def:consistent}). 

Furthermore, the following model parameters are also randomly selected: the velocity transformation matrices $\ls^{\lambda(i)|i}\vX_{\lambda(i)}$, which are body fixed constants, and the rigid body inertias\footnote{Random variable distribution was again uniform in the interval $[0,1]$. The inertia matrices is enforced symmetric and only positive definite matrices are considered.} $\bbM_{\cB i}$ 

We computed numerically the geometric linearization
in Proposition~\ref{proposition:linearization} 
by making use of (geometric) finite difference. 
Namely, we approximate 
$[\D_1 FD(\vH,\vs,\vv,\vr,\vtau)\circ \D L_H (I) ]_i$,
$[\D_2 FD(\vH,\vs,\vv,\vr,\vtau)]_k$,
$[\D_3 FD(\vH,\vs,\vv,\vr,\vtau)]_i$,
and 
$[\D_4 FD(\vH,\vs,\vv,\vr,\vtau)]_k$,
with $[X]_l$ denoting the $l$-th column of matrix $X$
and $i \in \{1 ... 6\}$ and $k \in \{1 ... n_J\}$ 
as 
\begin{align}
  & 
  \frac1\delta\left(FD(\vH\exp( (\Delta^{base}_i)^\wedge ),\vs,\vv,\vr,\vtau) - FD(\vH,\vs,\vv,\vr,\vtau)\right), \\
  & 
  \frac1\delta\left(FD(\vH,\vs+\Delta^{joint}_k,\vv,\vr,\vtau) - FD(\vH,\vs,\vv,\vr,\vtau)\right), \\
  & 
  \frac1\delta\left(FD(\vH,\vs,\vv+\Delta^{base}_i,\vr,\vtau) - FD(\vH,\vs,\vv,\vr,\vtau)\right), \\
  & 
  \frac1\delta\left(FD(\vH,\vs,\vv,\vr+\Delta^{joint}_k,\vtau) - FD(\vH,\vs,\vv,\vr,\vtau)\right),
\end{align}
where $\Delta^{base}_i \in \R^6$ is the perturbation vector for the moving base and $\Delta^{joint}_k \in \R^{n_J}$ for the joints.
The perturbation vectors $\Delta^{base}_i$ and $\Delta^{joint}_k$ are defined as zero-vectors, except the $i$-th or $k$-th index, which is equal to the perturbation scalar $\delta \in \R$ (for example, $\Delta^{base}_2 := [0; \delta ;0;0;0;0]$). We have employed the constant value  $\delta = 10^{-6}$. Adaptive finite difference perturbation or central finite difference could have been use to get even more accurate results, but as the main result was to get a confirmation that the correctness of the expressions and their algorithm implementation, this was deemed sufficient.

As numerical validation criteria, we define a maximum error and an average error between the left-trivialized linearization obtained via the approach described in Section~\ref{chp:algorithmicaspects} and the finite differences as (${idx} \in \{1,2,3,4\}$)
\begin{align} 
    e^{max}_{idx} 
    & :=
    max_{e \in E}\left|\frac{\D_{idx}FD^{Alg}(e) - \D_{idx}FD^{FinDiff}(e)}{avg|\D_{idx}FD^{Alg}|}\right| , 
    \label{eq:errormax} \\
    e^{avg}_{idx} 
    & :=
    avg_{e \in E}\left|\frac{\D_{idx}FD^{Alg}(e) - \D_{idx}FD^{FinDiff}(e)}{avg|\D_{idx}FD^{Alg}|}\right| . 
    \label{eq:erroravg}
\end{align}
where $E$ denotes the set of $100$ randomly selected combinations of configurations, geometric, and inertial properties, and $avg|\D_{idx} {FD}^{Alg}| = avg_{e \in E}|\D_{idx}{FD}^{Alg}(e)|$.

The following values where obtained from 
the set of 100 randomized combinations:
\begin{align}
    e^{max}_1 & = 4.1023\cdot10^{-5} ,
    &
    e^{avg}_1 & = 2.3560\cdot10^{-6} ,
\\
    e_{2}^{max} & = 4.6853\cdot10^{-3} ,
    & 
    e_{2}^{avg} & = 1.3604\cdot10^{-4} , 
\\
    e_{3}^{max} & = 1.8230\cdot10^{-5} ,
    & 
    e_{3}^{avg} & = 1.3021\cdot10^{-6} ,
\\
    e_{4}^{max} & = 1.5693\cdot10^{-4} ,
    & 
    e_{4}^{avg} & = 1.3766\cdot10^{-6} .
\end{align}
These results provide numerical evidence that our proposed method of computing the left-trivialized geometric linearization and its MATLAB implementation are correct. To provide further evidence, we report a parameter study showing how the mean of $e^{max}_{idx}$ and $e^{avg}_{idx}$, $idx \in \{1,2,3,4\}$, vary as a function of delta in Figure~\ref{fig:parametric_study}. The plot illustrates that as expected the error of the finite difference approximation is reduced as $\delta$ is reduced, up to the point where round off errors become dominant.

\begin{figure}[h!]
  \centering
  \includegraphics[scale=0.7]{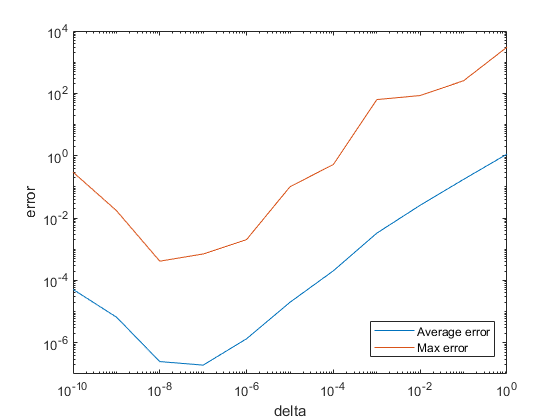}
  \caption{Parametric study with respect to $\delta$ of the  errors illustrating how the maximal and average error of the finite difference approximation is reducing as $\delta$ is reduced, up to the point where round off errors become dominant.}
  \label{fig:parametric_study}
\end{figure}

\section{Conclusions}
\label{chp:conclusions}
In this work, we have shown how to define a notion of mathematically elegant and singularity-free geometric linearization for moving-base robotic systems. Given this new formulation, we have also presented an algorithm to compute the exact linearization quantities in a computationally efficient manner. The algorithm has been derived by combining the proposed geometric linearization with modifications of the recursive dynamics algorithms~\cite{featherstone2008rigid}, widely used in robotics and computer graphics.

The algorithm results have been validated by comparing them against geometric finite differences, showing that the results are comparable, up to expected numerical error induced by finite differencing.

Future works include the implementation of the proposed linearization algorithms in a high performance language such as, e.g., C++ for real-time applications and its use for control, estimation, and optimal control planning of humanoid and quadruped robots.










\bibliographystyle{plain}
\bibliography{bibl} 

\medskip
Received xxxx 20xx; revised xxxx 20xx; early access xxxx 20xx.
\medskip

\end{document}